\documentclass{article} 
\usepackage{iclr2027_conference,times}

\usepackage{amsmath,amsfonts,bm}

\def\eqref#1{equation~\ref{#1}}

\def\1{\bm{1}}

\DeclareMathAlphabet{\mathsfit}{\encodingdefault}{\sfdefault}{m}{sl}
\SetMathAlphabet{\mathsfit}{bold}{\encodingdefault}{\sfdefault}{bx}{n}

\usepackage{amsthm}
\usepackage{thmtools}
\usepackage{thm-restate}
\usepackage[most]{tcolorbox}

\theoremstyle{plain}

\tcolorboxenvironment{restatable}{
  enhanced,
  breakable,
  colback=blue!3,
  colframe=gray!80,
  boxrule=1pt,
  arc=0pt,
  left=1.0mm,
  right=1.0mm,
  top=1.0mm,
  bottom=1.0mm,
  before skip=8pt,
  after skip=8pt
}

\usepackage{changepage} 
\usepackage{algorithm}
\usepackage{algorithmic}
\usepackage{capt-of}

\usepackage[utf8]{inputenc}
\usepackage{caption}
\usepackage{multirow}
\usepackage{subcaption}
\usepackage{pbox}
\usepackage{graphicx}
\usepackage{bm}
\usepackage{siunitx}
\usepackage{lipsum}
\usepackage{mathtools}
\usepackage[table]{xcolor}
\usepackage{adjustbox}
\usepackage{wrapfig}
\usepackage{stfloats}
\usepackage{amssymb}

\usepackage{array}
\usepackage{booktabs}
\usepackage{colortbl}
\newcolumntype{C}{>{\centering\arraybackslash}p{1.15cm}}

\usepackage{tabularx}
\usepackage{arydshln}
\usepackage{nicefrac}       
\usepackage{microtype}      
\usepackage{xcolor}         
\usepackage{paralist}
\usepackage[table]{xcolor}

\usepackage{newfloat}
\usepackage{listings}

\usepackage{hyperref}
\usepackage{url}

\title{Looped Reasoning Finishes Earlier\\ Than You Think}

\author{%
    Sumin Park\textsuperscript{\textnormal{1,2,3}},
    T. Konstantin Rusch\textsuperscript{\textnormal{1,2,3,4}\thanks{\scriptsize Correspondence to \href{mailto:tkrusch@tue.ellis.eu}{\texttt{tkrusch@tue.ellis.eu}}}~}\\
    {\small \textsuperscript{1}ELLIS Institute Tübingen,}
    {\small \textsuperscript{2}Max Planck Institute for Intelligent Systems,}\\ 
    {\small \textsuperscript{3}Tübingen AI Center,}
    {\small \textsuperscript{4}Liquid AI}
}

\iclrfinalcopy

\begin{document}

\maketitle

\lhead{Looped Reasoning Finishes Earlier Than You Think}

\vspace{-1em}
\begin{abstract}
Looped reasoning with dynamic depth can be viewed as relaxation toward solution attractors in a learned dynamical landscape. In practice, however, latent trajectories often oscillate, cycle, or drift rather than cleanly converge to stable fixed points. How these non-stationary dynamics contribute to reasoning and constrain convergence remains unclear. We show that slowly decaying oscillatory modes can impose their convergence rate on coupled latent dynamics, creating a bottleneck. Yet damping these modes indiscriminately can disrupt the reasoning computation they support. This tension motivates an architectural approach: introducing rotational degrees of freedom along which symmetry-generated rotational motion leaves the reasoning-relevant quotient dynamics unchanged by construction. We propose \textbf{QRR} (\textbf{Q}uotient \textbf{R}ecurrent \textbf{R}easoner), a gauge-equivariant recurrent architecture with an exact per-token phase-rotation symmetry. Motion along these symmetry directions leaves the quotient representation unchanged and is dynamically decoupled from reasoning-relevant dynamics. Empirically, QRR achieves stable and faster convergence to solution roots in quotient space, without forming a dominant slow oscillatory bottleneck, while matching or surpassing recent reasoning models on Sudoku, Maze, and Mini-ARC. Code is available at \href{https://github.com/camail-official/Quotient_Recurrent_Reasoner_QRR}{here}.
\end{abstract}

\section{Introduction}
\label{sec:intro}

Looped reasoners trade architectural depth for recurrent computation~\citep{dehghani2019universaltransformers,saunshi2025reasoninglatentthoughts,fan2025loopedtransformerslengthgeneralization}. A shared neural block is applied repeatedly to refine an internal state, allowing additional computation without increasing parameter count~\citep{NEURIPS2022_7f70331d,geiping2025scalingtesttimecomputelatent,wang2025hierarchicalreasoningmodel,jolicoeurmartineau2025morerecursivereasoningtiny}. Recent works make the dynamical interpretation explicit, viewing reasoning as the evolution of a latent state converging toward solution-aligned equilibria~\citep{blayney2026mechanistic,kim2026dynamicalphaseselectioncontrols}. Fixed-Point Reasoning Model (FPRM) uses convergence of the recurrent state to a fixed point as an intrinsic stopping criterion, Attractor Models define the prediction itself as the equilibrium of a learned refinement process, and Equilibrium Reasoners (EqR) explicitly learns task-conditioned attractors whose stable fixed points correspond to valid solutions, using multiple initializations to explore their attraction  basins~\citep{movahedi2026fixedpointreasonersstableadaptive,feinashley2026solveloopattractormodels,huang2026equilibriumreasonerslearningattractors}. In this view, reasoning is not just a sequence of refinements but a recurrent trajectory through the learned dynamical landscape.

However, viewing recurrent computation as convergence to a stable attractor does not capture the full range of dynamics it can exhibit. Trajectories can instead exhibit orbits, drifts, cyclic limiting behavior, and other structured forms of non-stationarity, rather than simply relaxing toward a single fixed point~\citep{geiping2025scalingtesttimecomputelatent,pappone2025twoscalelatentdynamicsrecurrentdepth,blayney2026mechanistic,kim2026dynamicalphaseselectioncontrols,yang2026stabilizingrecurrentdynamicstesttime}.
These studies reveal a wide range of latent geometries in recurrent models, yet how these different forms of oscillatory motion affect the underlying latent computation remains unexplored. 

The key distinction is therefore not the geometry of the motion itself, but whether and how it influences the recurrent computation. Around a solution root $z^\star$, where the recurrence admits a fixed linearization with the Jacobian $J=Df(z^\star)$, we establish a local coupling mechanism by separating an oscillatory spectral coordinate from the complementary latent dynamics and quantifying how the former drives the latter. When the oscillatory mode decays more slowly, this coupling transfers its rate to the complementary dynamics and delays settlement. We then test this mechanism causally in the original nonlinear recurrences of FPRM and HRM~\citep{wang2025hierarchicalreasoningmodel} in Section~\ref{sec:osc-coupling-empirical}. Across distinct oscillatory geometries, an initial perturbation of the identified mode produces latent responses and settling-time shifts that closely match the predictions of the local coupling model. This agreement directly links the convergence delay near solution roots to the coupled slow oscillatory mode.

\begin{wrapfigure}{r}{0.61\textwidth}
    \centering
    \scriptsize
    \vspace{-1.9em}
    \includegraphics[width=\linewidth]{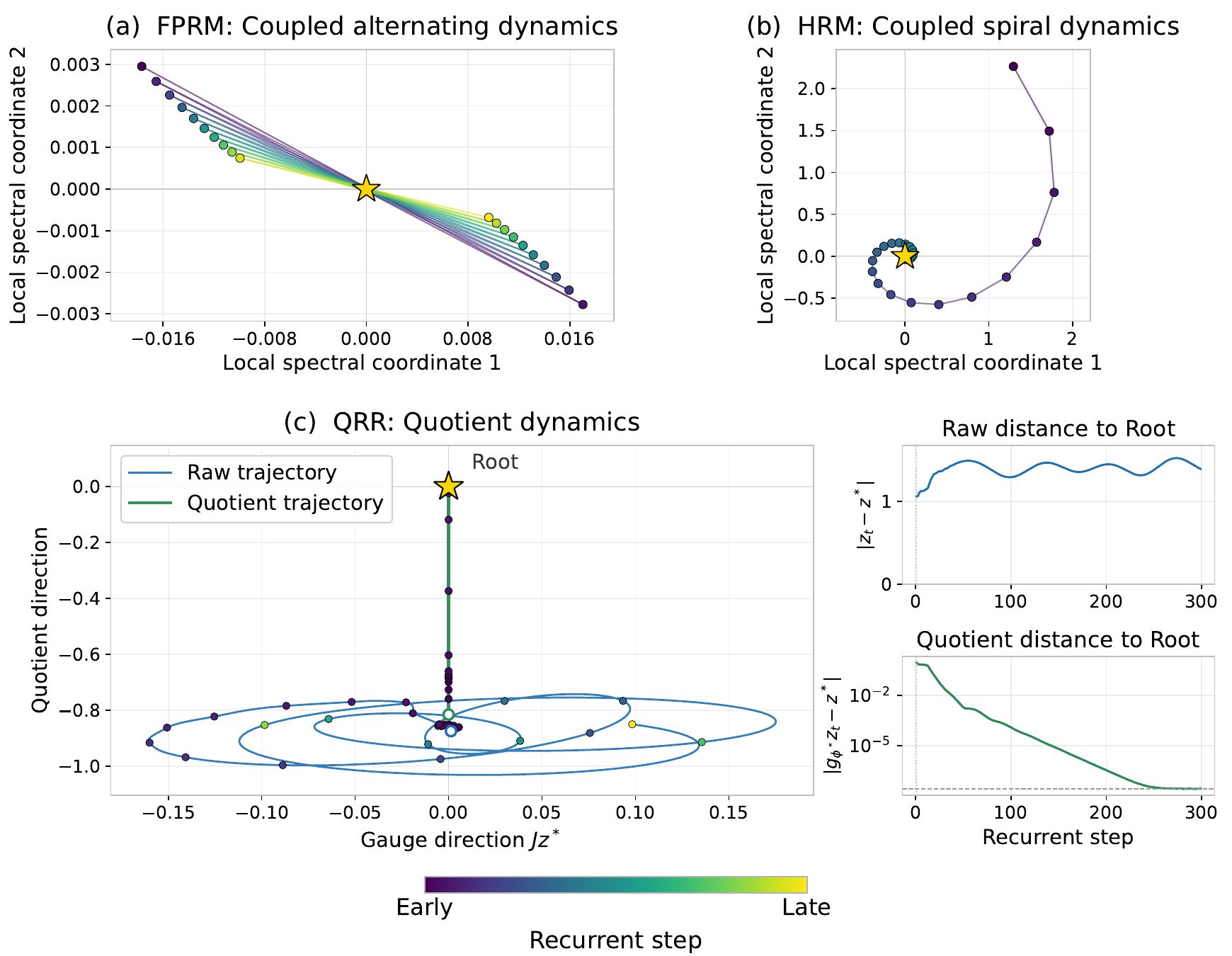}
    \vspace{-1.6em}
    \caption{
    \textbf{Oscillatory geometry and quotient convergence.}
    FPRM and HRM exhibit distinct alternating and spiral dynamics near solution roots. In QRR, raw motion persists while quotient trajectory contracts.
    }
    \label{fig:teaser}
    \vspace{-1.5em}
\end{wrapfigure}
Oscillatory appearance alone, however, does not reveal which motion can be safely removed. In FPRM, suppressing the alternating mode before semantic stabilization degrades task performance, showing that slow oscillatory motion can still support reasoning (Appendix~\ref{app:fprm-osc-suppression}). 
Thus, rather than identifying and damping particular oscillations, we instead redesign what counts as convergence by introducing a controlled representational redundancy: motion along rotational symmetry directions is irrelevant to the underlying quotient state by construction. We propose \textbf{QRR}, the \textbf{Q}uotient \textbf{R}ecurrent \textbf{R}easoner, which imposes an exact per-token rotational symmetry on the recurrent state, inducing a quotient representation invariant to this motion. Gauge-equivariant updates preserve this equivalence, allowing symmetry-generated motion to persist in the raw trajectory without driving or limiting the convergence of quotient dynamics.

As in Figure~\ref{fig:teaser}, FPRM and HRM exhibit distinct oscillatory geometries whereas QRR retains raw phase motion while its quotient trajectory converges. Oscillatory modes outside the gauge subspace remain free to participate in reasoning. Empirically, QRR also retains oscillatory modes in quotient space but does not develop the isolated dominant slow mode as in FPRM and HRM (Section~\ref{sec:spectral_bridge}).

\vspace{-1em}
\section{Background}
\vspace{-0.5em}

\subsection{Looped reasoning models}
\label{sec:bg-looped-reasoners}

\vspace{-0.5em}
Looped reasoning models increase effective computational depth by repeatedly applying
shared neural blocks to a latent state, allowing additional test-time computation without
increasing parameter count. Recent approaches can be broadly grouped by how they organize
this recurrent computation. General recurrent-depth models iterate a shared Transformer
block and scale reasoning by increasing the number of latent recurrences at inference
\citep{geiping2025scalingtesttimecomputelatent}. Puzzle-oriented recursive reasoners instead structure
the recurrence explicitly. HRM couples high- and low-level recurrent modules operating at
different frequencies~\citep{wang2025hierarchicalreasoningmodel}, whereas Tiny Reasoning Model (TRM) simplifies
this hierarchy to a single small network that repeatedly refines a latent reasoning state
and the current answer~\citep{jolicoeurmartineau2025morerecursivereasoningtiny}.
A complementary line organizes recurrent computation around equilibria or attractors:
FPRM uses convergence to a fixed point as an adaptive halting criterion
\citep{movahedi2026fixedpointreasonersstableadaptive}, EqR learns task-conditioned
attractors whose stable equilibria correspond to solutions
\citep{huang2026equilibriumreasonerslearningattractors}, and Attractor Models explicitly solve a learned
fixed-point refinement with implicit differentiation
\citep{feinashley2026solveloopattractormodels}.

\vspace{-1em}
\subsection{Gauge symmetry and quotient dynamics}
\label{sec:bg-gauge-quotient}
\vspace{-0.5em}

Gauge symmetry formalizes a local representational redundancy where changing the local
reference frame changes the coordinates used to represent a state, but not the
underlying object itself~\citep{pmlr-v97-cohen19d,bronstein2021geometric,gerken2023geometric}.
If a map $T$ is equivariant under the action of a group $G$, gauge-related states
$z\sim g\!\cdot\!z$ evolve consistently and can be identified as a single equivalence
class $[z]$ in the quotient space $\mathcal Z/G$~\citep{kondor2018generalization,cohen2019general}. The quotient collapses symmetry-related representations into a single state, so motion
along the symmetry orbit is ignored while all remaining state changes are preserved. This distinction also changes the notion of convergence. An equivariant system may continue moving along a symmetry orbit without changing its
quotient state. Such motion is captured by a {relative equilibrium},
$T(z^\star)=g^\star\!\cdot\!z^\star$, which is non-stationary in the original coordinates
but becomes a fixed point after quotienting, $[T(z^\star)]=[z^\star]$
\citep{krupa1990bifurcations,marsden1999introduction}.
In neural architectures, gauge equivariance provides a principled way to implement this invariance to local reference frames~\citep{weiler2021coordinate,wiersma2020cnns}.
Gauge-equivariant convolutions extend conventional equivariance from global transformations
to local gauge choices on manifolds~\citep{pmlr-v97-cohen19d}, and Gauge Equivariant Mesh
CNNs combine anisotropic convolution with parallel transport to communicate features
in different local frames~\citep{haan2021gauge}. The same principle has also been extended to attention, where the Gauge Equivariant Transformer constructs gauge-equivariant self-attention by comparing features across local coordinate systems after transporting them to a common frame~\citep{NEURIPS2021_e57c6b95}.

\vspace{-0.5em}
\subsection{Oscillatory and non-stationary dynamics in recurrent reasoning}
\label{sec:bg-oscillatory-dynamics}
\vspace{-0.3em}

Recent depth-recurrent language models exhibit structured latent trajectories that are substantially richer than monotone convergence to a fixed point, including multi-dimensional orbits and approximately one-dimensional drifts, particularly on tokens requiring extended computation
\citep{geiping2025scalingtesttimecomputelatent,pappone2025twoscalelatentdynamicsrecurrentdepth,kim2026dynamicalphaseselectioncontrols}.
Other mechanistic studies find cyclic fixed-point trajectories, where recurrent layers repeatedly trace a stable cycle in latent space, as well as persistent non-stationarity even after successive state differences become small
\citep{huang2026equilibriumreasonerslearningattractors,blayney2026mechanistic}.
Oscillatory behavior also appears in recurrent reasoners based on solution equilibria. FPRM, for example, reports oscillation around solution fixed points and introduces damping to stabilize these trajectories
\citep{movahedi2026fixedpointreasonersstableadaptive}.
More broadly, looped Transformers with the same architecture, objective, and task performance can settle into different dynamical phases, including bounded non-stationary motion generated by complex-conjugate Jacobian modes
\citep{kim2026dynamicalphaseselectioncontrols}.

What remains unclear is the dynamical role of this motion.
An observed orbit or alternating trajectory may support recurrent computation, or instead reflect a slowly decaying mode that unnecessarily prolongs latent
settlement after the remaining dynamics already converged.

\vspace{-0.8em}
\section{Rate-Limiting Slow Oscillatory Modes in Recurrent Reasoners}~\label{sec:oscillatory_coupling}
\vspace{-2.2em}

Motivated by oscillatory structures observed in recurrent reasoners, we study whether an
identified oscillatory mode can become rate-limiting for latent convergence.
Around solution roots, the nonlinear recurrence admits a stationary first-order linearization, allowing us to isolate an oscillatory spectral component from its
complementary latent dynamics.

We use a Sylvester decomposition as an analytic attribution tool to separate the latent state dynamics into its intrinsic evolution and the response induced by coupling
with the identified spectral coordinate.
This provides a direct way to quantify how the identified coordinate influences the remaining latent dynamics and their convergence behavior.
The decomposition itself is general
to any selected mode, particular oscillatory structure enters through local Jacobian block identified from the trained model.

\vspace{-0.7em}
\subsection{Local spectral decomposition and rate dominance}
\label{sec:theory}
\vspace{-0.5em}

We analyze how an oscillatory component of the local dynamics affects the
remaining latent state near a solution root. Let $z_{t+1}=f(z_t)$ denote the trained
nonlinear recurrence, and let $z^\star$ be a solution root satisfying
$f(z^\star)=z^\star$. For a nearby trajectory state, write the root-centered displacement $\delta_t=z_t-z^\star$. Since $z^\star$ is stationary, the same local Jacobian
governs subsequent steps as
\begin{equation}
\delta_{t+1}
=
J\delta_t+O(\|\delta_t\|^2),
\qquad
J=Df(z^\star).
\label{eq:empirical-local-linearization}
\end{equation}
This stationary linearization allows us to identify and track a fixed spectral component over recurrence.
Consider an oscillatory spectral component of $J$, and let
$U\in\mathbb R^{D\times m}$ be an orthonormal basis for the corresponding
left-invariant subspace of $J$, and let $S$ denote the restriction of $J$
to this subspace. Equivalently,
$J^\top U=US^\top$, with $S=U^\top J U$. Hence, defining a coordinate 
$u_t:=U^\top\delta_t$, the local dynamics satisfy
\vspace{-0.2em}
\begin{equation}
u_{t+1}=U^\top J\delta_t+O(\|\delta_t\|^2)
=Su_t+O(\|\delta_t\|^2).
\end{equation}
Thus $u_t$ is an autonomous spectral coordinate to first order, its next state update depends only on itself through the reduced Jacobian block $S$.
We represent the remaining displacement by the orthogonal residual $e_t$, 
\vspace{-0.2em}
\begin{equation}
u_t=U^\top\delta_t,
\qquad
e_t=Q\delta_t,
\qquad
Q=I-UU^\top,
\label{eq:local-coordinates}
\end{equation}
so that $\delta_t=Uu_t+e_t$ and $U^\top e_t=0$. Projecting the linearized
dynamics onto these two coordinates, $u_t$ and $e_t$, gives
(see full derivation in Appendix~\ref{app:localblock})
\begin{equation}
u_{t+1}=Su_t,
\qquad
e_{t+1}=Ae_t+Bu_t,
\qquad
A:=QJQ,
\quad
B:=QJU.
\label{eq:local-block-dynamics}
\end{equation}
Here $A$ governs the intrinsic evolution of the complementary state, whereas $B$
measures how the isolated spectral coordinate drives that complement.
The observed $e_t$ therefore mixes two effects: its intrinsic transient under
$A$ and the response transmitted from $u_t$ coupled by $B$. We separate these two distinct contributions to evolve $e_t$ using the following Sylvester decomposition.
\begin{restatable}[Exact Sylvester decomposition of a coupled spectral coordinate]{lemma}{coupledresponse}
\label{lem:coupled-response}
Consider the linear dynamics in Eq.~\ref{eq:local-block-dynamics}, with
$A\in\mathbb R^{d\times d}$, $S\in\mathbb R^{m\times m}$, and
$B\in\mathbb R^{d\times m}$, and assume
$\sigma(A)\cap\sigma(S)=\emptyset$.
Then there exists a unique matrix $P\in\mathbb R^{d\times m}$ satisfying $PS-AP=B$.
The complementary state admits the exact decomposition
\vspace{-0.3em}
\begin{equation}
e_t
=
A^t(e_0-Pu_0)+PS^tu_0.
\label{eq:exact-decomposition}
\end{equation}
\end{restatable}

Lemma~\ref{lem:coupled-response} separates the complementary state $e_t$ into two subsystems.
The first term, $A^t(e_0-Pu_0)$, is its intrinsic transient, whereas
$PS^tu_0$ is the response transmitted from the coupled spectral coordinate $u_t$.  
The lemma is not specific to oscillatory dynamics, it provides a general
framework for attributing the effect of one autonomous spectral component on the rest of a coupled linear system.

\vspace{-0.7em}
\paragraph{Slow oscillatory specialization.}
We now specialize this decomposition to the reduced block $S=U^\top J U$ observed in the actual recurrent reasoners, where the isolated mode is oscillatory and decays
more slowly than the complementary dynamics. When the measured $S$ has
a complex-conjugate pair with modulus $r$ and frequency $\omega$, its real
two-dimensional dynamics admit
\begin{equation}
\sigma(S)=\{re^{\pm i\omega}\},
\qquad
S=M(rQ_\omega)M^{-1},
\qquad
0<\omega<\pi,
\label{eq:osc-rotation}
\end{equation}
where $Q_\omega$ is the planar rotation through $\omega$ and $M$ is real and
invertible.
When $\rho(A)<r$, the intrinsic transient in
Eq.~\ref{eq:exact-decomposition} decays faster than the coupled response
$PS^tu_0$.
For a trajectory $e=(e_t)_{t\geq0}$, define
\begin{equation}
\tau_\delta(e)=\inf\{t\in\mathbb N_0:\|e_k\|_2\leq\delta\ \text{for all }k\geq t\},
\qquad
\mathcal E(e)=\inf_{c\in\mathbb R^d}\limsup_{t\to\infty}\|e_t-c\|_2.
\label{eq:osc-settlement}
\end{equation}
\vspace{-0.7em}
\begin{restatable}[Rate inheritance from a slow oscillatory mode]{theorem}{oscillatorycoupling}
\label{thm:oscillatory-coupling}
Under the setting of Lemma~\ref{lem:coupled-response}, suppose $S$ satisfies
Eq.~\ref{eq:osc-rotation}, with $\rho(A)<r\leq1$, $B\neq0$, and $u_0\neq0$.
If $r<1$, then $e_t\to0$,
\vspace{-0.5em}
\begin{equation}
\limsup_{t\to\infty}\|e_t\|_2^{1/t}=r>\rho(A),
\label{eq:osc-inherited-rate}
\end{equation}
and
\vspace{-0.5em}
\begin{equation}
\tau_\delta(e)=\frac{\log(1/\delta)}{-\log r}+O(1),\qquad \delta\downarrow0.
\label{eq:osc-settling-asymptotic}
\end{equation}
If $r=1$, then $e_t$ remains bounded but does not converge to any constant,
with
\begin{equation}
\mathcal E(e)\geq\frac{\|u_0\|_2\|B\|_F}{\sqrt{2}\,\kappa_2(M)(1+\|A\|_2)}>0.
\label{eq:osc-persistent-floor}
\end{equation}
\end{restatable}
Theorem~\ref{thm:oscillatory-coupling} shows that a slower coupled mode transfers its rate to the complementary dynamics, delaying settlement when $r<1$ and preventing stationary convergence when $r=1$. It covers general two-dimensional spiral modes, while $M=I$ recovers a scaled rotation and $S=-r$ gives the alternating flip case of Corollary~\ref{cor:alternating-coupling}. All proofs are provided in Appendix~\ref{app:theorems}.


\vspace{-0.5em}
\subsection{Causal verification in nonlinear recurrent dynamics}
\label{sec:osc-coupling-empirical}
\vspace{-0.5em}

We now use this framework to ask whether the slow oscillatory modes identified in trained recurrent reasoners actually exert the predicted rate-limiting effect on their
nonlinear trajectories.
For each solution root, we identify a leading oscillatory eigenspace of the local Jacobian,
construct its left-invariant basis $U$, and obtain $S,A,$ and $B$ matching the decomposition in Section~\ref{sec:theory}.
We focus on cases satisfying the rate-dominance condition
$\rho(A)<r\le1$ (Appendix~\ref{app:empvalrho}).

\vspace{-0.2em}
To test the contribution of this mode in the original nonlinear dynamics, we take a
state from the native trajectory once it has entered the local linear regime around
the solution root $z^\star$.
Writing
\vspace{-0.2em}
\begin{equation}
\delta_0
:=
z_{t_0}-z^\star
=
e_0+Uu_0,
\qquad
e_0=Q\delta_0,
\qquad
u_0=U^\top\delta_0,
\label{eq:empirical-native-decomposition}
\end{equation}
separates the displacement into the identified oscillatory coordinate $u_0$ and the
remaining local state $e_0$.
We then vary only the amplitude of $u_0$ by adjusting $\alpha$,
\vspace{-0.2em}
\begin{equation}
z_{t_0}^{(\alpha)}
=
z^\star+e_0+\alpha Uu_0,
\qquad
\alpha\in\left\{1,\tfrac12,0\right\},
\label{eq:empirical-intervention}
\end{equation}
while keeping $e_0$ exactly fixed.
Here $\alpha=1$ is the native state, $\alpha=0.5$ attenuates the oscillatory
coordinate, and $\alpha=0$ removes it at the intervention point.
This is the only modification and all trajectories subsequently evolve using the
original trained nonlinear recurrence $f$.
Because all arms share the same complementary state $e_0$, the linearized rollout gives
$\widehat e_k^{(\alpha)}=Q\delta_k = QJ^k(e_0+\alpha Uu_0)$ (Eq.~\ref{eq:local-coordinates}), so subtracting the $\alpha=0$
arm isolates the response to the oscillatory coordinate. Using the decomposition in Eq.~\ref{eq:exact-decomposition}, this prediction can be further written equivalently as
\vspace{-0.2em}
\begin{equation}
\widehat{\Delta e}_k^{(\alpha)}
=
\widehat e_k^{(\alpha)}-\widehat e_k^{(0)}
=
\alpha QJ^kUu_0
=
\alpha\left(PS^k-A^kP\right)u_0.
\label{eq:general-response-prediction}
\end{equation}
This response prediction is determined entirely by the local Jacobian and the native
displacement at $t_0$, with no quantities fit to the subsequent nonlinear rollout.

\vspace{-0.5em}
\paragraph{Alternating oscillatory mode in FPRM.}
We apply the intervention protocol to the rate-dominant oscillatory mode identified in FPRM.
As illustrated in Figure~\ref{fig:teaser}, this mode appears as an approximately one-dimensional alternating flip, with
$S=-r$ and $\rho(A)<r\le1$.
The corresponding response prediction is
(derivations in Appendix~\ref{cor:alternating-coupling})
\vspace{-0.3em}
\begin{equation}
\widehat{\Delta e}_k^{(\alpha)}
=
\alpha\left[P(-r)^k-A^kP\right]u_0
=
\alpha P(-r)^ku_0+o(r^k).
\label{eq:fprm-response-prediction}
\end{equation}
Thus, changing $\alpha$ scales the leading response while leaving its asymptotic
decay rate unchanged: $\alpha=\tfrac12$ halves its leading amplitude, whereas
$\alpha=0$ removes the $r^k$ contribution.
We compare this closed-form local prediction directly with the response measured
from the original nonlinear FPRM recurrence.
For $\alpha=0$, the linearized counterfactual reduces to
$\widehat e_k^{(0)}=A^ke_0$, allowing us to additionally test whether removing
the oscillatory coordinate accelerates root-relative settlement.

\begin{figure*}[h]
\vspace{-1em}
    \centering
    \small
    \includegraphics[width=0.90\textwidth]{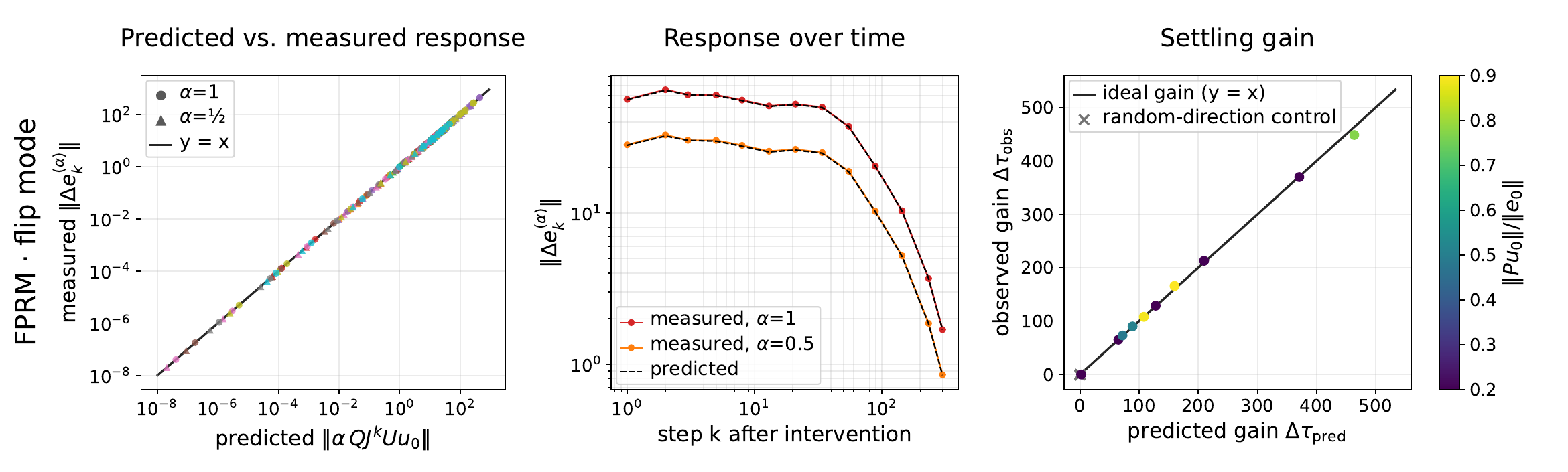}
    \vspace{-0.7em}
    \caption{
    \textbf{Local oscillatory-coupling theory predicts the response and settling cost in the original FPRM recurrence.}
    We control only the Jacobian-identified oscillatory coordinate once with
    $\alpha\in\{0,0.5,1\}$ and then resume unmodified recurrent updates. HRM results in Appendix~\ref{app:hrm-spiral-validation}.
    }
    \vspace{-0.8em}
    \label{fig:theoryfit}
\end{figure*}

Figure~\ref{fig:theoryfit} tests the oscillatory-coupling mechanism at three
levels: (i) the response induced in the complementary state, (ii) its evolution over
recurrent time, and (iii) the resulting cost in settling time.

The left panel measures how a perturbation of the oscillatory coordinate subsequently propagates into the complementary latent dynamics over $k$ recurrent steps. For each intervention strength $\alpha$ and rollout step $k$, the horizontal
axis shows the prediction from the root Jacobian, while the vertical axis shows the
corresponding change measured from the original nonlinear rollout:
\begin{equation}
\underbrace{
\|\Delta e_k^{(\alpha)}\|_2
=
\left\|
Q\!\left(
z_{t_0+k}^{(\alpha)}-z_{t_0+k}^{(0)}
\right)
\right\|_2
}_{\text{actual nonlinear rollout}}
\qquad\text{vs.}\qquad
\underbrace{
\|\widehat{\Delta e}_k^{(\alpha)}\|_2
=
\|\alpha QJ^kUu_0\|_2
}_{\text{local linear prediction}} .
\label{eq:predicted-measured-response}
\end{equation}
Agreement with the identity line shows that the response induced in the
complementary state is quantitatively predicted by the local Jacobian.

The middle panel tests whether this prediction remains accurate throughout the
subsequent nonlinear rollout. Consistent with
Eq.~\ref{eq:fprm-response-prediction}, the measured response follows the predicted
alternating temporal envelope, including the proportional reduction under
$\alpha=0.5$.

The right panel tests whether the same local mode produces the predicted delay in
convergence. Using the root-relative settling criterion
$\tau_\delta(e)=\inf\!\left\{
k:\|e_j\|\le\delta\ \text{for all }j\ge k
\right\}$
from Eq.~\ref{eq:osc-settlement}, we define
$\hat{\tau}^{(\alpha)}:=\tau_\delta(\hat e^{(\alpha)})$ and
$\tau^{(\alpha)}:=\tau_\delta(e^{(\alpha)})$ using the same sustained-window
implementation. The corresponding gains from removing the oscillatory coordinate are
$\Delta\tau_{\mathrm{pred}}
=\hat{\tau}^{(1)}-\hat{\tau}^{(0)}$
and
$\Delta\tau_{\mathrm{obs}}
=\tau^{(1)}-\tau^{(0)}$.
The identity-line agreement therefore confirms that the local prediction quantitatively matches the actual settling-time gain from removing the identified oscillatory component.

Together, the three columns test the causal chain
$u_0
\;\rightarrow\;
\Delta e_k
\;\rightarrow\;
\Delta\tau_\delta$ in the original FPRM. The root Jacobian predicts how the oscillatory coordinate
drives the complementary latent state, and removing that coordinate accelerates
settlement while preserving the decoded output.
Extended analysis for distinct two-dimensional spiral mode in HRM is provided in Appendix~\ref{app:hrm-spiral-validation}, where we obtain the same prediction--measurement agreement.

\vspace{-1em}
\section{Quotient dynamics convergence}
\label{sec:quotient_dynamics}

\vspace{-0.5em}
Our analysis in Section~\ref{sec:oscillatory_coupling} shows that slowly decaying oscillatory
modes can limit convergence speed in a locally linearized regime near a solution root.
This does not, however, justify suppressing oscillatory motion throughout the entire reasoning trajectory.
Rotational or alternating dynamics may still participate in the reasoning computation, and removing them indiscriminately can disrupt reasoning.
Indeed, persistently suppressing an alternating trajectory component in FPRM before output stabilization
degrades task performance (see Appendix~\ref{app:fprm-osc-suppression}).
Rather than attempting to directly correct the oscillatory bottlenecks identified above, QRR takes a complementary architectural approach by introducing an exact token-local gauge symmetry.

Importantly, QRR does not assume any correspondence between these gauge directions and the slow modes identified in other reasoners. QRR learns its own recurrent trajectory. The symmetry instead introduces token-local rotational degrees of freedom whose motion is separated from the task-relevant quotient dynamics. How this symmetry changes the learned dynamics is therefore an empirical question: QRR still retains oscillatory modes in quotient space, but does not develop the isolated dominant slow mode observed in FPRM or HRM.

\vspace{-0.5em}
\subsection{Token-local gauge symmetry}
\label{sec:gauge_action}
\vspace{-0.3em}

Let the recurrent state consist of $S$ token states
$h=(h_1,\ldots,h_S)$, with $h_s\in\mathbb{C}^{d}$,
together with momentum states $m_s\in\mathbb{C}^{d}$.
In our implementation, each real-valued token state $y_s\in\mathbb{R}^{2d}$ is reinterpreted as $h_s\in\mathbb{C}^{d}$ through the fixed pairing
$h_{s,c}=y_{s,2c-1}+i\,y_{s,2c}$, $c=1,\ldots,d$.
The input representation $x$ is gauge-neutral.
We introduce the group
$G=U(1)^S$,
with one phase $\phi_s$ for each token.
For
$g_\phi=(e^{i\phi_1},\ldots,e^{i\phi_S})\in G$, the group action is
\begin{equation}
g_\phi\cdot(h,m)
=
\bigl(
(e^{i\phi_s}h_s)_{s=1}^S,\,
(e^{i\phi_s}m_s)_{s=1}^S
\bigr),
\qquad
g_\phi\cdot x=x.
\label{eq:gauge_action}
\end{equation}

\vspace{-0.3em}
Thus, all complex channels within a token share a common phase, while different tokens may rotate independently.
Equivalently, in real coordinates, let $\mathcal{J}$ denote the block-diagonal infinitesimal generator of the phase rotation, so that the gauge tangent associated with token $s$ is generated by $\mathcal{J}y_s$ (or, equivalently, by $ih_s$ in complex coordinates).
We require the recurrent map $T$ to be $G$-equivariant while the prediction readout $\mathcal{R}$ is $G$-invariant,
\begin{equation}
T(g\cdot z;x)=g\cdot T(z;x),\qquad
\mathcal{R}(g\cdot z;x)=\mathcal{R}(z;x),\qquad z=(h,m).
\label{eq:gauge_equivariance}
\end{equation}
Hence, states that differ only by token-local phase are identified as the same task-relevant state, and motion along the corresponding gauge directions is invisible to the prediction readout.

\vspace{-1em}
\subsection{Quotient dynamics}
\label{sec:quotient_definition}
\vspace{-0.5em}

The symmetry identifies all states on the same group orbit,
$[z]=G\cdot z=\{g\cdot z:g\in G\}$,
yielding the quotient state space
$\mathcal{Q}=\mathcal{Z}/G$.
Because the recurrent map is $G$-equivariant, it induces a well-defined map on the quotient,
\vspace{-0.4em}
\begin{equation}
\overline{T}([z];x)=[T(z;x)].
\label{eq:quotient_map}
\end{equation}
The quotient removes only symmetry-generated degrees of freedom, while preserving all remaining latent variation.
Accordingly, a state may satisfy
$T(z^\star;x)=g^\star\!\cdot z^\star$ for some $g^\star\in G$,
making $z^\star$ a \emph{relative equilibrium} where the unreduced trajectory may continue along the group orbit, while
$\overline{T}([z^\star];x)=[z^\star]$ is an ordinary fixed point in the quotient.
Moreover, exact equivariance ensures that a perturbation along the gauge orbit remains gauge-tangent after the recurrent update, and therefore produces no change in the quotient state.

At a relative equilibrium, the quotient analogue of
Eq.~\ref{eq:local-block-dynamics} takes the form 
$e_{t+1}=A_Qe_t+B_Gu_t$ with $B_G=0$, so this symmetry-generated motion cannot
transmit its timescale into the quotient dynamics.
We formalize this non-forcing property in Proposition~\ref{prop:gauge-nonforcing}.

\vspace{-0.2em}
\paragraph{Quotient-based halting.}
\begin{wrapfigure}{r}{0.29\textwidth}
    \vspace{-1.5em}
    \centering
    \includegraphics[width=0.29\textwidth]{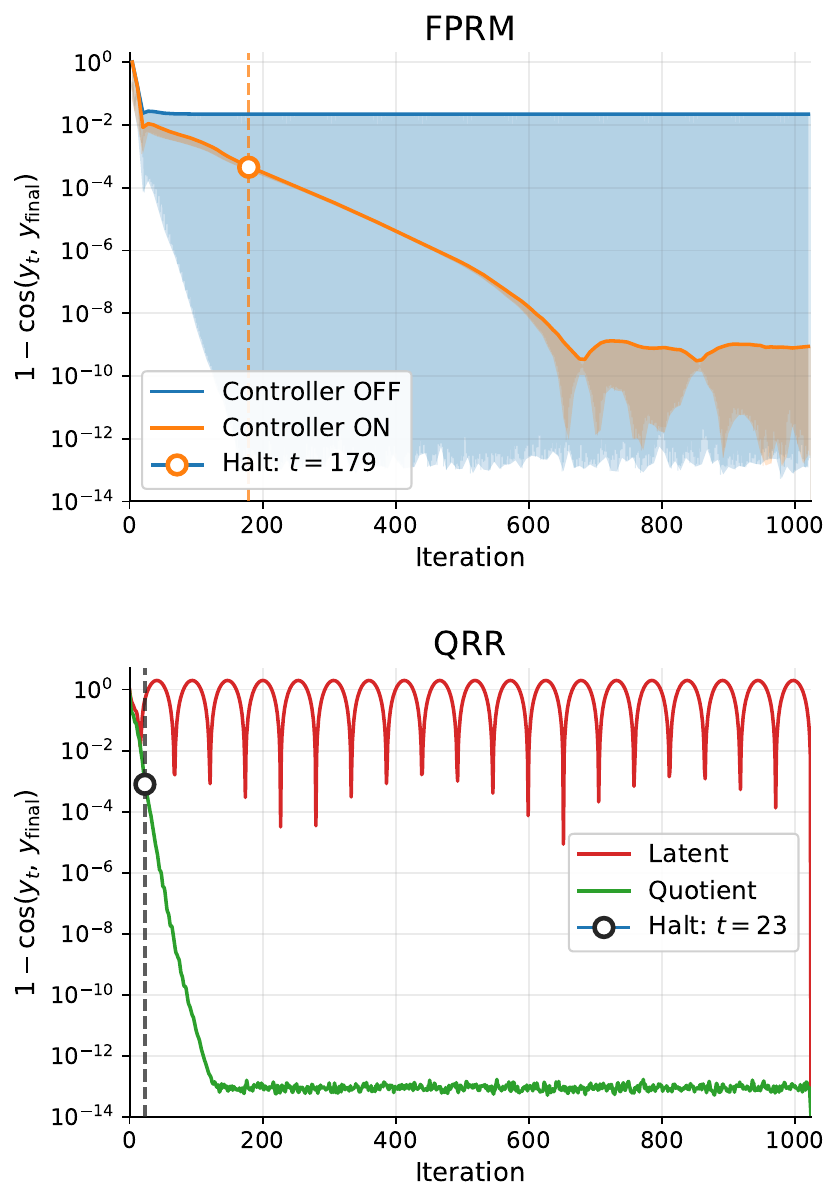}
    \vspace{-1.8em}
    \caption{{Convergence and halting in quotient space.}}
    \label{fig:quotient-halting}
    \vspace{-0.8em}
\end{wrapfigure}
The quotient view changes what it means for recurrent reasoning to have converged.
At a relative equilibrium, the state may continue moving along its gauge orbit while its quotient state remains fixed.
Thus, requiring raw-state stationarity can mistake symmetry-generated phase motion for unfinished computation even after the quotient state has converged.
Figure~\ref{fig:quotient-halting} illustrates this distinction between raw-state and quotient convergence.
In FPRM, residual oscillatory motion prevents the raw latent state from becoming stationary until it is damped by the step-size controller, making its fixed-point halting criterion sensitive to this motion.
In QRR, by contrast, the raw state may continue moving along its gauge orbit after the quotient trajectory has converged to numerical precision, allowing convergence to be detected substantially earlier.
We therefore measure convergence modulo the gauge,
\vspace{-0.7em}
\begin{equation}
r_G(z)
=
\inf_{g\in G}
\frac{\left\|T(z;x)-g\cdot z\right\|}
{\left\|z\right\|+\epsilon}.
\label{eq:gauge_residual}
\end{equation}
This residual vanishes when the current state and its next update differ only by a gauge transformation, while remaining sensitive to motion in the quotient.
At inference, QRR halts once $r_G$ remains below a calibrated threshold for a prescribed persistence window.

\subsection{Gauge-equivariant recurrent block}
\label{sec:gauge_architecture}
\vspace{-0.3em}
 
\begin{wrapfigure}{r}{0.29\textwidth}
    \vspace{-2em}
    \centering
    \includegraphics[width=\linewidth]{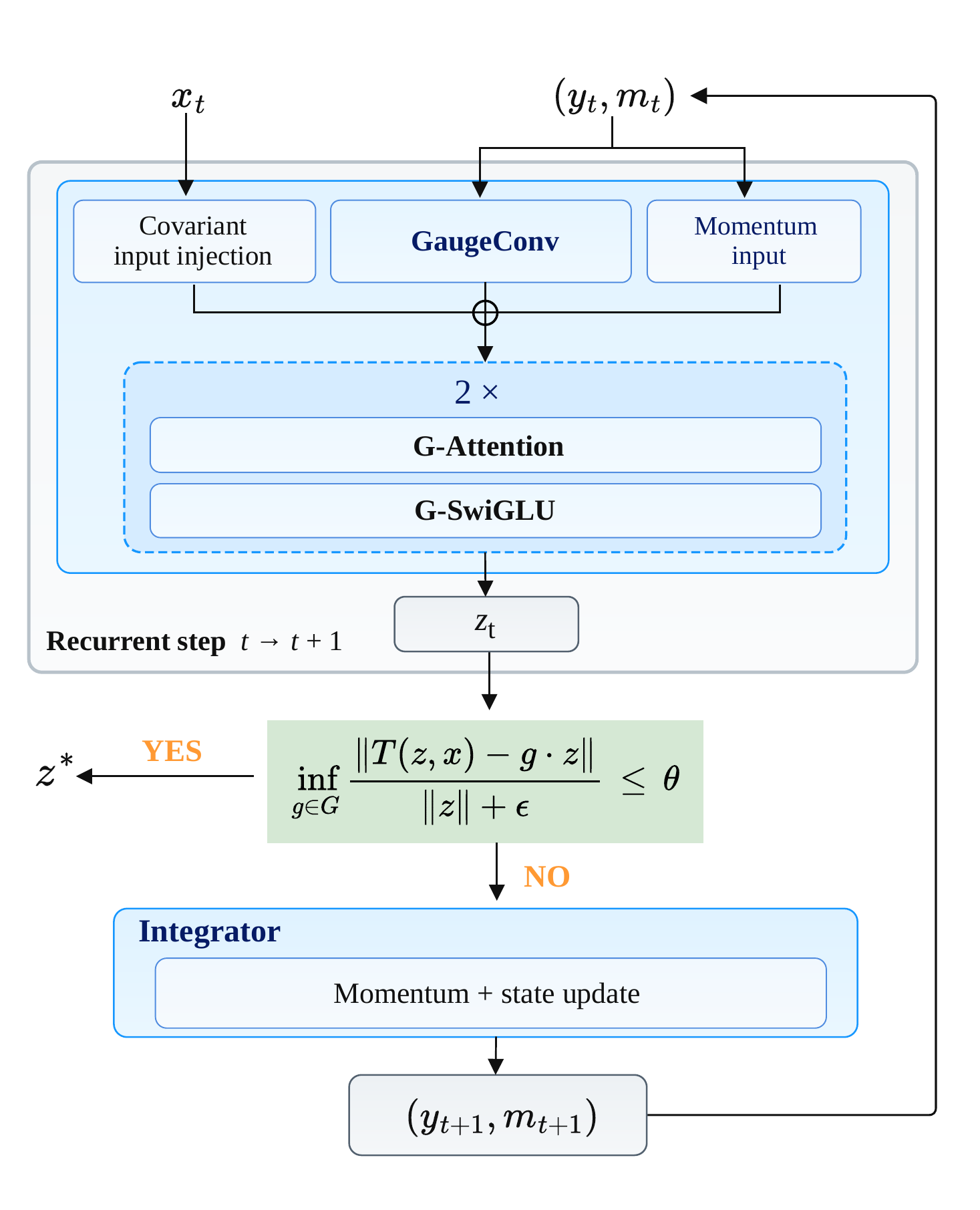}
    \vspace{-1em}
    \caption{{Overview of the QRR recurrent block.}}
    \label{fig:qrr_architecture}
    \vspace{-2em}
\end{wrapfigure}
QRR adopts the recurrent Transformer block of FPRM as its backbone, while replacing its internal operations with gauge-equivariant counterparts.
Under the token-local symmetry in Equations~\ref{eq:gauge_action}--\ref{eq:gauge_equivariance}, the latent state and momentum transform covariantly with the local token phase, while attention weights and prediction readouts remain gauge-invariant.
Figure~\ref{fig:qrr_architecture} provides an overview of the overall QRR architectural workflow.
Full architectural details are in Appendix~\ref{app:full_arc}.

\vspace{-0.6em}
\paragraph{Relative transport.}
Because each token has its own local phase, features from different tokens must be aligned before they can be mixed.
We associate each token with a covariant phase $\gamma_s\in U(1)$, satisfying
$\gamma_s\mapsto e^{i\phi_s}\gamma_s$, and define the relative transport
$U_{s\leftarrow t}
=
\gamma_s\gamma_t^\ast$.
Under a pure gauge rotation,
$U_{s\leftarrow t}\mapsto e^{i(\phi_s-\phi_t)}U_{s\leftarrow t}$,
providing the relative phase required to transport features from token $t$ into the local frame of token $s$.
We use this transport for both token mixing and self-attention.

\vspace{-0.6em}
\paragraph{Gauge-equivariant attention.}
Queries, keys, and values are produced by complex-linear maps.
Using the relative transport, we define the attention scores as
\begin{equation}
a_{st}
=
\frac{1}{\sqrt d}
\operatorname{Re}
\left[
U_{s\leftarrow t}\,
q_s^\dagger k_t
\right],
\label{eq:gauge_attention_score}
\end{equation}
which are gauge-invariant.
Values are likewise transported into the local frame of token $s$ before aggregation,
$o_s
=
W_O
\sum_t
\operatorname{softmax}_t(a_{st})\,
U_{s\leftarrow t}v_t,$
so that $o_s$ transforms covariantly with token $s$.

\vspace{-0.6em}
\paragraph{Gauge-covariant input injection and feed-forward layers.}
Because the input $x_s$ is gauge-neutral, it cannot be added directly to the recurrent state.
We construct a covariant phase carrier
\begin{equation}
D_s
=
\frac{\langle W_d x_s,h_s\rangle}
{|\langle W_d x_s,h_s\rangle|+\epsilon},
\qquad
h_s
\leftarrow
\alpha_0 h_s+\beta_0 D_s x_s,
\label{eq:gauge_input}
\end{equation}
where $D_s\mapsto e^{i\phi_s}D_s$, so that the injected input transforms in the local frame of token $s$.
For the feed-forward block, we use an invariant gate with an equivariant value branch,
\begin{equation}
\operatorname{G\text{-}SwiGLU}(h)
=
W_{\mathrm{down}}
\left[
\operatorname{SiLU}(|W_g h|)
\odot
W_v h
\right].
\label{eq:gauge_swiglu}
\end{equation}

\vspace{-1em}
\paragraph{Momentum-carrying recurrent dynamics.}
\label{sec:momentum_dynamics}
We use a momentum-carrying second-order recurrence to evolve the latent state.
At each recurrent step, the gauge-equivariant trunk produces a state proposal
$z_t=\Phi_\theta(y_t;x,m_t)$, defining the residual $F_t=z_t-y_t$.
We update the momentum as
\begin{equation}
\vspace{-0.2em}
p_t
=
\beta m_t+(1-\beta)F_t,
\qquad
\beta=\sigma(b_{\mathrm H}),
\label{eq:momentum}
\end{equation}
where $b_{\mathrm H}\in\mathbb R$ is a learned global scalar.
The recurrent state is then updated by
\begin{equation}
\vspace{-0.2em}
y_{t+1}=y_t+p_t,
\qquad
m_{t+1}=p_t.
\label{eq:momentum_update}
\end{equation}
The carried momentum is also fed back into the recurrent trunk at the next step.

\vspace{-0.2em}
\subsection{Spectral Dominance and Convergence in Reasoning Dynamics}
\label{sec:spectral_bridge}
\vspace{-0.6em}
We next examine how these dynamical constraints are reflected in the learned spectra of QRR.
\begin{wrapfigure}{r}{0.7\textwidth}
\centering
\small
\vspace{-0.8em}
\includegraphics[width=\linewidth]{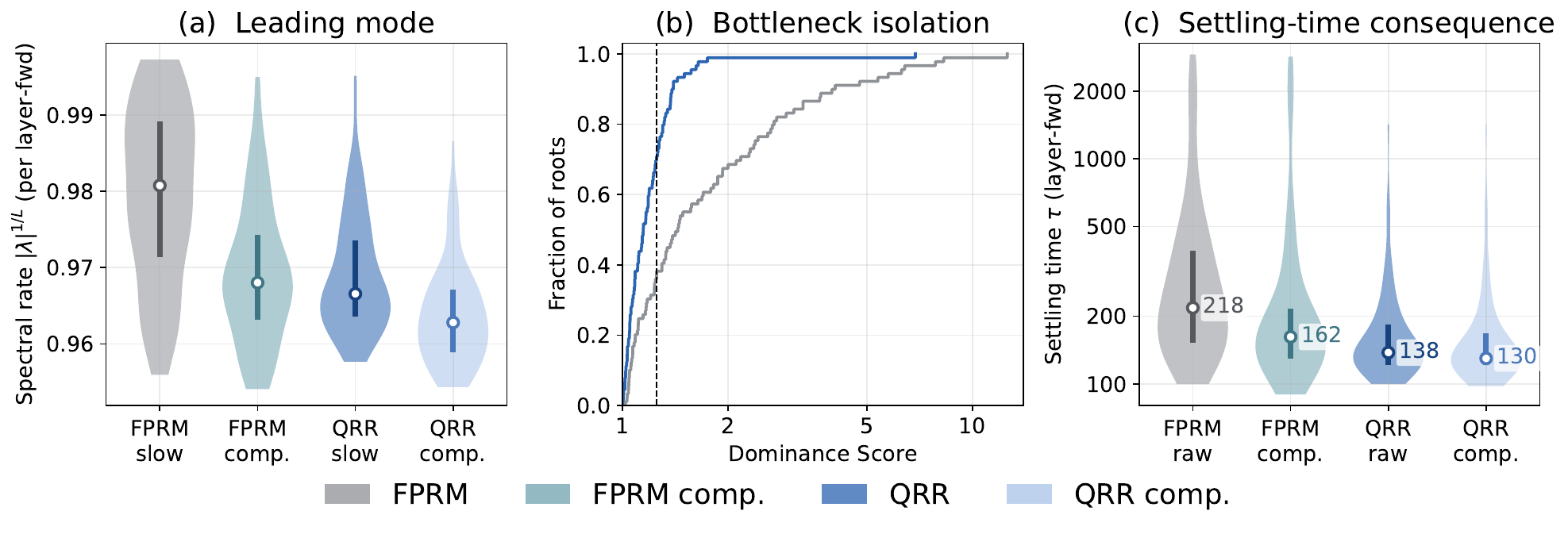}
\vspace{-2em}
\caption{
{Spectral analysis and settling near solution roots for FPRM and QRR. Results for HRM in Appendix~\ref{app:spectral_bridge}}
}
\label{fig}
\vspace{-0.8em}
\end{wrapfigure}
We characterize the local spectral structure of FPRM and QRR around certified solution roots that decode to the ground truth. FPRM roots are ordinary fixed points of the latent dynamics, whereas QRR roots are relative equilibria, fixed points in quotient space. We linearize FPRM in latent space and QRR in quotient space, and report spectral rates as $|\lambda|^{1/L}$, where $L$ is the number of recurrent layer-forwards (lfwd) per native step. Following Section~\ref{sec:oscillatory_coupling}, we decompose each system into its leading spectral mode and complementary dynamics. For FPRM, the leading mode is slow oscillatory block identified in Section~\ref{sec:osc-coupling-empirical}. For QRR, both are derived from quotient space. 

Figure~\ref{fig} (a) shows QRR operates in a faster spectral regime with less clear separation between its leading and remaining dynamics.
Figure~\ref{fig} (b) quantifies this separation using the dominance score $D=\log r_{\mathrm{comp}}/\log r_{\mathrm{slow}}$, $D=1$ indicates matched timescales, while larger $D$ indicates a more isolated slow leading mode. FPRM frequently exhibits strong spectral dominance, whereas QRR remains much closer to $D=1$ despite retaining some weaker oscillatory modes. Panel~(c) shows the corresponding time settlement. QRR settles earlier than FPRM, while factoring out the leading mode accelerates FPRM ($218\!\rightarrow\!162$) but the effect remains only modest in QRR ($138\!\rightarrow\!130$). Thus, QRR operates in a faster convergence regime with weaker dependence on a single slow leading mode.

\vspace{-0.3em}
\section{Experiments}~\label{sec:experiments}
\vspace{-3em}
\paragraph{Datasets and baselines}
We evaluate on four reasoning benchmarks covering constraint satisfaction, path finding,
and abstract transformation. \textbf{Sudoku-Extreme} consists of challenging partially
filled $9\times9$ Sudoku puzzles with unique solutions; following the standard HRM
protocol, we train on 1{,}000 base puzzles with $1{,}000\times$ augmentation
and evaluate using exact sequence accuracy on the full test set
~\citep{wang2025hierarchicalreasoningmodel}.
\textbf{Maze-Hard} is a $30\times30$ shortest-path benchmark with 1{,}000 training and
1{,}000 test mazes. We follow the FPRM setup with no augmentation
~\citep{movahedi2026fixedpointreasonersstableadaptive}.
For \textbf{Mini-ARC}, a compact $5\times5$ ARC variant for abstract
rule induction~\citep{MiniARC_NeurIPSW2022,huang2026equilibriumreasonerslearningattractors},
we use the released tasks with $100$ data augmentations per task.
We additionally evaluate on \textbf{Maze-Unique}, following the EqR construction.
Unlike Maze-Hard, where multiple optimal paths may correspond to a single supervised target, Maze-Unique uses perfect tree-structured mazes with a unique start--goal solution~\citep{MiniARC_NeurIPSW2022}. This provides a cleaner convergence analysis, avoiding the ambiguous solution geometry that can hinder formation of a single stable basin~\citep{huang2026equilibriumreasonerslearningattractors}. We compare against HRM, TRM, EqR, and FPRM, spanning hierarchical recurrence,
recursive refinement, learned attractors, and fixed-point adaptive halting.


\vspace{-0.5em}
\subsection{Results}
\vspace{-0.4em}

Table~\ref{tab:task_performance_runtime} summarizes both reasoning performance and computational cost. In Table~\ref{tab:task_performance_runtime} (a), we prioritize matched evaluation from an official checkpoint when available, otherwise use the official result, and reproduce the baseline under its native setup. HRM, TRM, and EqR use their own fixed inference budgets, whereas FPRM and QRR halt adaptively according to fixed-point residuals evaluated in the latent and quotient space, respectively. For the latter, we report halt-or-cap performance with caps of 35k recurrent steps on Sudoku and 3k on the Maze tasks. QRR achieves the strongest performance on Sudoku, Maze-Unique, and Mini-ARC while remaining competitive on Maze-Hard.
Table~\ref{tab:task_performance_runtime} (b) reports corresponding runtime cost in single A100 GPU on Sudoku-Extreme. QRR takes approximately $1.5\times$ higher LFWD latency and
$1.3\times$ higher memory than FPRM, reflecting the additional cost of its
complex-valued gauge-equivariant operations and explicit phase transport.
Nevertheless, faster convergence outweighs this overhead. The median end-to-end
rollout time decreases from $5.4$ s to $1.9$ s, and the p95 from $168.1$ s to
$60.0$ s ($2.8\times$ faster). Thus, despite its higher per-LFWD cost, QRR completes full reasoning rollout faster than FPRM.

\begin{table*}[h]
\centering
\caption{
\textbf{Reasoning performance and runtime.}
(a) exact accuracy on Sudoku-Extreme, Maze-Hard, and Maze-Unique, and task-level Pass@1 on Mini-ARC.
Models with adaptive halting (FPRM and QRR) report halt-or-cap results.
Superscripts denote: \textcolor{black}{$^{1}$ evaluation matched from an official checkpoint,
$^{2}$ official report, and
$^{3}$ full reproduction following the native setup.}
(b) Comparison of per-layer-forward (LFWD) computational cost and end-to-end rollout time (s).
}
\label{tab:task_performance_runtime}

\begin{minipage}[t]{0.655\textwidth}
\centering
\vspace{-0.5em}
\small
{(a) Reasoning performance}

\vspace{0.35em}
\begin{adjustbox}{max width=\linewidth}
\begin{tabular}{lcccccc}
\toprule
\textbf{Model}
& \textbf{Params}
& \textbf{Sudoku-Extreme}
& \multicolumn{2}{c}{\textbf{Maze}}
& \textbf{Mini-ARC} \\
\cmidrule(lr){4-5}
& (M)
& Exact $\uparrow$
& Hard Exact $\uparrow$
& Unique Exact $\uparrow$
& Pass@1 $\uparrow$ \\
\midrule

HRM
& 27
& 55.4$^1$
& 74.5$^2$
& 51.9$^3$
& 56.4$^3$ \\

TRM
& 7
& 75.7$^1$
& 85.3$^2$
& 58.0$^3$
& 40.9$^3$ \\

EqR
& 7
& 89.9$^1$
& --
& 88.9$^2$
& 41.6$^3$ \\

FPRM
& 7
& 91.5$^1$
& \textbf{86.4}$^1$
& --
& -- \\

\textcolor{black}{FPRM (Rep.)}
& 7
& \textcolor{black}{83.7}$^3$
& \textcolor{black}{85.3}$^3$
& \textcolor{black}{94.7}$^3$
& \textcolor{black}{57.1}$^3$ \\

\midrule
\rowcolor{blue!5}
\textbf{QRR}
& \textbf{7}
& \textbf{97.9}
& 85.8
& \textbf{97.0}
& \textbf{62.4} \\

\bottomrule
\end{tabular}
\end{adjustbox}
\end{minipage}
\hfill
\begin{minipage}[t]{0.325\textwidth}
\centering
\vspace{-0.7em}
\small
{(b) Runtime comparison}

\vspace{0.32em}
\scriptsize
\setlength{\tabcolsep}{3pt}
\renewcommand{\arraystretch}{1.10}

\begin{adjustbox}{max width=\linewidth}
\begin{tabular}{lcc}
\toprule
\textbf{Setting}
& \textbf{QRR}
& \textbf{FPRM} \\
\midrule

\multicolumn{3}{l}{\textit{Per LFWD ($B=512$)}} \\

Runtime (ms)
& 34.7
& 22.3 \\

Peak memory (GB)
& 2.88
& 2.14 \\

\midrule

\multicolumn{3}{l}{\textit{End-to-end rollout}} \\

Coverage (\%)
& 98.0
& 100.0 \\

Halt step, p50 / p95
& 27 / 865
& 122 / 3,769 \\

Runtime (s), p50 / p95
& 1.9 / 60.0
& 5.4 / 168.1 \\

\bottomrule
\end{tabular}
\end{adjustbox}
\end{minipage}

\end{table*}

\begin{table*}[h]
\centering
\small
\setlength{\tabcolsep}{3.5pt}
\renewcommand{\arraystretch}{1.12}
\vspace{-0.5em}
\caption{
\textbf{Convergence and adaptive halting.}
{(a)} Root convergence on Sudoku-Extreme.
{(b)} Adaptive halting of FPRM and QRR across tasks.
Coverage is the fraction halted within finite cap.
}
\label{tab:convergence_halting}

\hspace{-3.5em}
\begin{minipage}[t]{0.45\textwidth}
\centering
\vspace{-0.3em}
{(a) Root convergence on Sudoku-Extreme}
\vspace{0.7em}

\scriptsize
\begin{tabular}{@{}lcc@{}}
\toprule
\textbf{Model}
& ${\tau_{\mathrm{root}}}$
& \shortstack{{Correct root found}} \\
& \scriptsize median [IQR] $\downarrow$
& \scriptsize (\%) $\uparrow$ \\
\midrule

HRM
& 1,032 [744, 1,272]
& 55.2 \\


EqR
& 672 [546, 966]
& 86.1\\

FPRM
& 234 [152, 396]
& 90.4\\

\textbf{QRR}
& \textbf{122 [108, 148]}
& 89.6\\

\bottomrule
\end{tabular}
\end{minipage}
\hspace{-1em}
\begin{minipage}[t]{0.54\textwidth}
\centering
\vspace{-0.4em}
{(b) Adaptive halting (cap Sudoku 35k, Maze-Unique 3k)}
\vspace{0.3em}

\scriptsize
\begin{tabular}{@{}llcccc@{}}
\toprule
\textbf{Task}
& \textbf{Model}
& \shortstack{\textbf{Coverage}\\(\%) $\uparrow$}
& \shortstack{$\tau_{\mathrm{halt}}$\\ median [IQR] $\downarrow$}
& \shortstack{\textbf{Acc.@halt}\\(\%) $\uparrow$}
& \shortstack{\textbf{False halt}\\(\%) $\downarrow$} \\
\midrule

\multirow{2}{*}{Sudoku-Extreme}
& FPRM
& 100.0
& 122 [39, 295]
& 91.1
& 8.9 \\

& \textbf{QRR}
& {98.0}
& 27 [23, 63]
& {99.9}
& {0.1} \\

\midrule

\multirow{2}{*}{Maze-Unique}
& FPRM
& 96.4
& 36 [27, 76]
& 98.2
& 1.8 \\

& \textbf{QRR}
& 98.0
& 42 [32, 61]
& 98.5
& 1.5 \\
\bottomrule
\end{tabular}
\end{minipage}
\end{table*}

Table~\ref{tab:convergence_halting} (a) evaluates root convergence on
Sudoku-Extreme using Anderson acceleration applied to each trained recurrent map.
Full solver, certification, and compute-normalization details are provided in
Appendix~\ref{app:rootconvanalys}.
\emph{Correct root found} is the fraction of examples for which the solver finds a root
decoding to the ground-truth solution, while $\tau_{\mathrm{root}}$ is the first sustained
entry of the native trajectory into a neighborhood of that root, measured in recurrent
layer-forward equivalents. For QRR, distance is measured in quotient space after removing
token-local gauge motion. QRR reaches solution roots significantly earlier than the baselines while maintaining a root-finding rate comparable to FPRM.
Table~\ref{tab:convergence_halting} (b) evaluates deployed adaptive halting under finite
rollout caps. \emph{Coverage} is the fraction halted within the cap, while
$\tau_{\mathrm{halt}}$ reports the median halting step and IQR over halted examples.
Among these, \emph{Acc.@halt} is the fraction with an exactly correct prediction, and
\emph{False halt} the fraction halted incorrectly. QRR halts earlier with
fewer false halts on Sudoku, while achieving comparable halting depth with higher coverage
and accuracy on Maze-Unique.

\paragraph{Ablations}~\label{sec:ablations}
\begin{wraptable}{r}{0.61\textwidth}
    \centering
    \vspace{-0.4em}
    \setlength{\tabcolsep}{4pt}
    \renewcommand{\arraystretch}{1.10}

    \resizebox{\linewidth}{!}{%
    \begin{tabular}{lcccc}
        \toprule
        \textbf{Variant}
        & {Exact} (\%) $\uparrow$
        & {Coverage} (\%) $\uparrow$
        & {False halt} (\%) $\downarrow$
        & ${\tau_{\mathrm{halt}}}$ [IQR] $\downarrow$ \\
        \midrule

        \multicolumn{5}{l}{\textit{Training}} \\
        \rowcolor{blue!6}
        \textbf{QRR}
        & 97.9 & 98.0 & 0.1 & 27 [23, 63] \\


        global $U(1)$
        & 91.3 & 0.0 & -- & CAP (never fires) \\

        first order w/o momentum
        & 87.8 & 79.6 & 0.1 & 28 [22, 44] \\

        FPRM trunk + momentum
        & 77.0 & 100 & 23.0 & 12,650 [12,522, 12,871] \\




        \bottomrule
    \end{tabular}%
    }

    \vspace{-0.5em}
    \scriptsize
    \caption{
    \textbf{QRR ablations on Sudoku-Extreme.}
    }
    \label{tab:qrr_ablation}
    \vspace{-2.0em}
\end{wraptable}
We ablate the main structural components of QRR.
Restricting the token-local $U(1)^S$ symmetry to a single global $U(1)$ preserves moderate task accuracy but never triggers halting, indicating that a shared global phase is too restrictive to account for the residual token-wise phase motion that persists after the decoded answer settles.
Removing momentum and using first-order dynamics disrupts both exact accuracy and coverage, while leaving the faster halting time and quotient halting precision of trajectories that do settle largely intact.
Finally, replacing the G-equivariant trunk with FPRM trunk while applying momentum sharply degrades accuracy and leads to unreliable convergence. All examples eventually halt, but only after longer trajectories and with a high false-halt rate.

\section{Conclusion}
\vspace{-0.6em}

We studied how non-stationary latent dynamics constrain convergence in recurrent
reasoners. Our analysis shows that a slowly decaying oscillatory mode can transmit its
rate to coupled latent dynamics and become the bottleneck to settlement, while direct
suppression of such motion can also disrupt the reasoning computation. This motivates
QRR, which does not eliminate oscillations based on their dynamical appearance, but instead introduces an exact token-local gauge symmetry under which symmetry-generated motion leaves the quotient state unchanged by construction. Exact gauge equivariance prevents this motion from forcing the quotient dynamics, allowing convergence and halting to be defined modulo the
symmetry. 
Across various reasoning tasks, QRR achieves strong reasoning performance
while reaching solution roots much earlier and supporting reliable adaptive
halting. These results suggest that convergence in recurrent reasoning is better understood
in terms of task-relevant quotient space than raw latent space.


\subsubsection*{Acknowledgments}
This work was supported by the Hector Foundation.




\bibliography{iclr2027_conference}

@misc{wang2025hierarchicalreasoningmodel,
  title={Hierarchical Reasoning Model},
  author={Guan Wang and Jin Li and Yuhao Sun and Xing Chen and Changling Liu and Yue Wu and Meng Lu and Sen Song and Yasin Abbasi Yadkori},
  year={2025},
  eprint={2506.21734},
  archivePrefix={arXiv},
  primaryClass={cs.AI},
  url={https://arxiv.org/abs/2506.21734}
}

@misc{jolicoeurmartineau2025morerecursivereasoningtiny,
  title={Less is More: Recursive Reasoning with Tiny Networks},
  author={Alexia Jolicoeur-Martineau},
  year={2025},
  eprint={2510.04871},
  archivePrefix={arXiv},
  primaryClass={cs.LG},
  url={https://arxiv.org/abs/2510.04871}
}

@misc{feinashley2026solveloopattractormodels,
  title={Solve the Loop: Attractor Models for Language and Reasoning},
  author={Jacob Fein-Ashley and Paria Rashidinejad},
  year={2026},
  eprint={2605.12466},
  archivePrefix={arXiv},
  primaryClass={cs.LG},
  url={https://arxiv.org/abs/2605.12466}
}

@misc{movahedi2026fixedpointreasonersstableadaptive,
      title={Fixed-Point Reasoners: Stable and Adaptive Deep Looped Transformers}, 
      author={Sajad Movahedi and Vera Milovanović and Shlomo Libo Feigin and Alexander Theus and Thomas Hofmann and Valentina Boeva and T. Konstantin Rusch and Antonio Orvieto},
      year={2026},
      eprint={2606.18206},
      archivePrefix={arXiv},
      primaryClass={cs.AI},
      url={https://arxiv.org/abs/2606.18206}, 
}

@misc{yang2026stabilizingrecurrentdynamicstesttime,
      title={Stabilizing Recurrent Dynamics for Test-Time Scalable Latent Reasoning in Looped Language Models}, 
      author={Xiao-Wen Yang and Ziyu Han and Xi-Hua Zhang and Wen-Da Wei and Jie-Jing Shao and Lan-Zhe Guo and Yu-Feng Li},
      year={2026},
      eprint={2605.26733},
      archivePrefix={arXiv},
      primaryClass={cs.LG},
      url={https://arxiv.org/abs/2605.26733}, 
}

@inproceedings{NEURIPS2022_7f70331d,
 author = {Bansal, Arpit and Schwarzschild, Avi and Borgnia, Eitan and Emam, Zeyad and Huang, Furong and Goldblum, Micah and Goldstein, Tom},
 booktitle = {Advances in Neural Information Processing Systems},
 doi = {10.52202/068431-1471},
 editor = {S. Koyejo and S. Mohamed and A. Agarwal and D. Belgrave and K. Cho and A. Oh},
 pages = {20232--20242},
 publisher = {Curran Associates, Inc.},
 title = {End-to-end Algorithm Synthesis with Recurrent Networks: Extrapolation without Overthinking},
 url = {https://proceedings.neurips.cc/paper_files/paper/2022/file/7f70331dbe58ad59d83941dfa7d975aa-Paper-Conference.pdf},
 volume = {35},
 year = {2022}
}

@misc{blayney2026mechanistic,
  title         = {A Mechanistic Analysis of Looped Reasoning Language Models},
  author        = {Blayney, Hugh and Arroyo, {\'A}lvaro and Obando-Ceron, Johan and Castro, Pablo Samuel and Courville, Aaron and Bronstein, Michael M. and Dong, Xiaowen},
  year          = {2026},
  eprint        = {2604.11791},
  archivePrefix = {arXiv}
}

@misc{kim2026dynamicalphaseselectioncontrols,
      title={Dynamical phase selection controls compute scaling in looped transformers}, 
      author={Gunn Kim},
      year={2026},
      eprint={2608.26556},
      archivePrefix={arXiv},
      primaryClass={cond-mat.dis-nn},
      url={https://arxiv.org/abs/2608.26556}, 
}

@misc{huang2026equilibriumreasonerslearningattractors,
      title={Equilibrium Reasoners: Learning Attractors Enables Scalable Reasoning}, 
      author={Benhao Huang and Zhengyang Geng and Zico Kolter},
      year={2026},
      eprint={2605.21488},
      archivePrefix={arXiv},
      primaryClass={cs.LG},
      url={https://arxiv.org/abs/2605.21488}, 
}

@misc{geiping2025scalingtesttimecomputelatent,
      title={Scaling up Test-Time Compute with Latent Reasoning: A Recurrent Depth Approach}, 
      author={Jonas Geiping and Sean McLeish and Neel Jain and John Kirchenbauer and Siddharth Singh and Brian R. Bartoldson and Bhavya Kailkhura and Abhinav Bhatele and Tom Goldstein},
      year={2025},
      eprint={2502.05171},
      archivePrefix={arXiv},
      primaryClass={cs.LG},
      url={https://arxiv.org/abs/2502.05171}, 
}

@misc{pappone2025twoscalelatentdynamicsrecurrentdepth,
      title={Two-Scale Latent Dynamics for Recurrent-Depth Transformers}, 
      author={Francesco Pappone and Donato Crisostomi and Emanuele Rodolà},
      year={2025},
      eprint={2509.23314},
      archivePrefix={arXiv},
      primaryClass={cs.LG},
      url={https://arxiv.org/abs/2509.23314}, 
}

@article{weiler2021coordinate,
  title   = {Coordinate Independent Convolutional Networks: Isometry and Gauge Equivariant Convolutions on Riemannian Manifolds},
  author  = {Weiler, Maurice and Forr{\'e}, Patrick and Verlinde, Erik and Welling, Max},
  journal = {arXiv preprint arXiv:2106.06020},
  year    = {2021}
}

@article{wiersma2020cnns,
  title     = {CNNs on Surfaces using Rotation-Equivariant Features},
  author    = {Wiersma, Ruben and Eisemann, Elmar and Hildebrandt, Klaus},
  journal   = {ACM Transactions on Graphics},
  volume    = {39},
  number    = {4},
  year      = {2020},
  doi       = {10.1145/3386569.3392437},
  publisher = {ACM}
}

@inproceedings{kondor2018generalization,
  title     = {On the Generalization of Equivariance and Convolution in Neural Networks to the Action of Compact Groups},
  author    = {Kondor, Risi and Trivedi, Shubhendu},
  booktitle = {Proceedings of the 35th International Conference on Machine Learning},
  pages     = {2747--2755},
  year      = {2018},
  volume    = {80},
  series    = {Proceedings of Machine Learning Research},
  publisher = {PMLR}
}

@article{cohen2019general,
  title   = {A General Theory of Equivariant CNNs on Homogeneous Spaces},
  author  = {Cohen, Taco S. and Geiger, Mario and Weiler, Maurice},
  journal = {arXiv preprint arXiv:1811.02017},
  year    = {2019}
}

@book{marsden1999introduction,
  title     = {Introduction to Mechanics and Symmetry: A Basic Exposition of Classical Mechanical Systems},
  author    = {Marsden, Jerrold E. and Ratiu, Tudor S.},
  series    = {Texts in Applied Mathematics},
  volume    = {17},
  publisher = {Springer},
  year      = {1999}
}

@article{bronstein2021geometric,
  title   = {Geometric Deep Learning: Grids, Groups, Graphs, Geodesics, and Gauges},
  author  = {Michael M. Bronstein and Joan Bruna and Taco Cohen and Petar Veli{\v{c}}kovi{\'c}},
  year    = {2021},
  journal = {arXiv preprint arXiv:2104.13478}
}

@inproceedings{haan2021gauge,
  title     = {Gauge Equivariant Mesh CNNs: Anisotropic convolutions on geometric graphs},
  author    = {Pim De Haan and Maurice Weiler and Taco Cohen and Max Welling},
  year      = {2021},
  booktitle = {International Conference on Learning Representations},
  url       = {https://openreview.net/forum?id=Jnspzp-oIZE}
}

@InProceedings{pmlr-v97-cohen19d,
  title     = {Gauge Equivariant Convolutional Networks and the Icosahedral {CNN}},
  author    = {Cohen, Taco and Weiler, Maurice and Kicanaoglu, Berkay and Welling, Max},
  booktitle = {Proceedings of the 36th International Conference on Machine Learning},
  pages     = {1321--1330},
  year      = {2019},
  editor    = {Chaudhuri, Kamalika and Salakhutdinov, Ruslan},
  volume    = {97},
  series    = {Proceedings of Machine Learning Research},
  publisher = {PMLR},
  url       = {https://proceedings.mlr.press/v97/cohen19d.html}
}

@article{krupa1990bifurcations,
  author  = {Krupa, Martin},
  title   = {Bifurcations of Relative Equilibria},
  journal = {SIAM Journal on Mathematical Analysis},
  volume  = {21},
  number  = {6},
  pages   = {1453--1486},
  year    = {1990},
  doi     = {10.1137/0521081}
}

@article{gerken2023geometric,
  title   = {Geometric Deep Learning and Equivariant Neural Networks},
  author  = {Gerken, Jan E. and Aronsson, Jimmy and Carlsson, Oscar and
             Linander, Hampus and Ohlsson, Fredrik and Petersson, Christoffer and
             Persson, Daniel},
  journal = {Artificial Intelligence Review},
  volume  = {56},
  pages   = {14605--14662},
  year    = {2023},
  doi     = {10.1007/s10462-023-10502-7}
}

@inproceedings{NEURIPS2021_e57c6b95,
  author={He, Lingshen and Dong, Yiming and Wang, Yisen and Tao, Dacheng and Lin, Zhouchen},
  booktitle={Advances in Neural Information Processing Systems},
  pages={27331--27343},
  title={Gauge Equivariant Transformer},
  volume={34},
  year={2021}
}

@inproceedings{MiniARC_NeurIPSW2022,
  title={{Playgrounds for Abstraction and Reasoning}},
  author={Kim, Subin and Phunyaphibarn, Prin and Ahn, Donghyun and Kim, Sundong},
  year={2022},
  web={https://bit.ly/ARC-GIST},
  booktitle={NeurIPS Workshop on neuro Causal and Symbolic AI}
}

@inproceedings{dehghani2019universaltransformers,
  title     = {Universal Transformers},
  author    = {Dehghani, Mostafa and Gouws, Stephan and Vinyals, Oriol and Uszkoreit, Jakob and Kaiser, Lukasz},
  booktitle = {International Conference on Learning Representations},
  year      = {2019}
}

@inproceedings{saunshi2025reasoninglatentthoughts,
  title     = {Reasoning with Latent Thoughts: On the Power of Looped Transformers},
  author    = {Saunshi, Nikunj and Dikkala, Nishanth and Li, Zhiyuan and Kumar, Sanjiv and Reddi, Sashank J.},
  booktitle = {International Conference on Learning Representations},
  year      = {2025},
  url       = {https://openreview.net/forum?id=din0lGfZFd}
}

@inproceedings{fan2025loopedtransformerslengthgeneralization,
  title     = {Looped Transformers for Length Generalization},
  author    = {Fan, Ying and Du, Yilun and Ramchandran, Kannan and Lee, Kangwook},
  booktitle = {International Conference on Learning Representations},
  year      = {2025},
  url       = {https://openreview.net/forum?id=2edigk8yoU}
}
\bibliographystyle{iclr2027_conference}

\clearpage

\appendix

\section{Appendix}

\subsection{Theoretical Proofs}~\label{app:theorems}

\coupledresponse*

\begin{proof}
Define the Sylvester operator
\[
\mathcal L(X):=XS-AX,
\qquad
\mathcal L:\mathbb R^{d\times m}\to\mathbb R^{d\times m}.
\]
Using vectorization,
\begin{equation}
\operatorname{vec}(\mathcal L(X))
=
\left(S^\top\otimes I_d-I_m\otimes A\right)\operatorname{vec}(X).
\label{eq:sylvester-vectorized}
\end{equation}
The eigenvalues of
$S^\top\otimes I_d-I_m\otimes A$
are
\[
\mu-\lambda,
\qquad
\mu\in\sigma(S),\quad \lambda\in\sigma(A).
\]
Because $\sigma(A)\cap\sigma(S)=\emptyset$, none of these eigenvalues is zero.
Hence the matrix in~Eq.~\ref{eq:sylvester-vectorized} is invertible, so
$\mathcal L$ is bijective.
Therefore, for every $B$, there exists a unique
$P\in\mathbb R^{d\times m}$ satisfying
\[
PS-AP=B.
\]

Now define the reparameterized auxiliary variable using $P$,
\begin{equation}
w_t:=e_t-Pu_t.
\end{equation}
Using
$e_{t+1}=Ae_t+Bu_t$,
$u_{t+1}=Su_t$, and
$PS-AP=B$, we obtain
\begin{align}
w_{t+1}
&=e_{t+1}-Pu_{t+1} \\
&=Ae_t+Bu_t-PSu_t \\
&=Ae_t-APu_t \\
&=A(e_t-Pu_t)
=Aw_t.
\end{align}
Thus
\[
w_t=A^tw_0=A^t(e_0-Pu_0).
\]
Since $u_t=S^tu_0$, substituting back into
$e_t=w_t+Pu_t$ gives
\[
e_t
=
A^t(e_0-Pu_0)+PS^tu_0,
\]
which proves~Eq.~\ref{eq:exact-decomposition}.
\end{proof}

\oscillatorycoupling*

\begin{proof}
Since
\[
S=M(rQ_\omega)M^{-1}
\]
and every eigenvalue of $S$ has modulus $r>\rho(A)$, we have
\[
\sigma(A)\cap\sigma(S)=\emptyset.
\]
Hence Lemma~\ref{lem:coupled-response} gives a unique solution $P$ of
$PS-AP=B$ and the exact decomposition
\begin{equation}
e_t=A^t(e_0-Pu_0)+PS^tu_0.
\label{eq:osc-proof-decomposition}
\end{equation}
Since $B\neq0$, the Sylvester equation implies $P\neq0$.
Define
\[
q:=M^{-1}u_0,\qquad C:=PM,\qquad a:=\|q\|_2.
\]
Invertibility of $M$ implies $a>0$ and $C\neq0$.
Moreover,
\[
S^t=r^tMQ_\omega^tM^{-1},
\]
so
\[
PS^tu_0=r^tCQ_\omega^tq.
\]
This change of oscillator coordinates leaves the complementary state
$e_t$ and its Euclidean norm unchanged.
Choose $\nu$ such that
\[
\rho(A)<\nu<r.
\]
There exists $C_\nu<\infty$ satisfying
\[
\|A^t\|_2\leq C_\nu\nu^t,\qquad t\geq0.
\]
Writing
\[
\xi_t:=A^t(e_0-Pu_0),\qquad p_t:=CQ_\omega^tq,
\]
we obtain
\begin{equation}
e_t=r^tp_t+\xi_t,\qquad
\|\xi_t\|_2=O(\nu^t)=o(r^t),
\label{eq:osc-leading-response}
\end{equation}
where $\xi_t$ takes the A-transient part and $r^tp_t$ corresponds to the oscillatory response. We first establish a uniform two-step lower bound on the oscillatory
response. This accounts for possible cancellation at individual steps.
Let
\[
v_t:=\frac{Q_\omega^tq}{a} \in \mathbb{R}^2,
\qquad
V_t:=\begin{pmatrix}v_t & Q_\omega v_t\end{pmatrix} \in \mathbb{R}^{2\times2}.
\]
Because $Q_\omega$ is orthogonal,
\[
\|v_t\|_2=\|Q_\omega v_t\|_2=1,
\qquad
\langle v_t,Q_\omega v_t\rangle=\cos\omega.
\]
Thus
\[
V_t^\top V_t=
\begin{pmatrix}
1 & \cos\omega\\
\cos\omega & 1
\end{pmatrix}
\]
has eigenvalues $1\pm|\cos\omega|$.
Moreover, because $v_t$ and $Q_\omega v_t$ are non-collinear for
$0<\omega<\pi$, the Gram matrix $V_t^\top V_t$ is positive definite.
The same eigenvalues occur in $V_tV_t^\top$, giving
\[
(1-|\cos\omega|)I_2
\preceq V_tV_t^\top
\preceq(1+|\cos\omega|)I_2.
\]
Since
\[
p_t=aCv_t,\qquad p_{t+1}=aCQ_\omega v_t,
\]
we have
\[
\|p_t\|_2^2+\|p_{t+1}\|_2^2
=a^2\|CV_t\|_F^2
=a^2\operatorname{tr}(CV_tV_t^\top C^\top).
\]
Consequently,
\begin{equation}
a^2(1-|\cos\omega|)\|C\|_F^2
\leq
\|p_t\|_2^2+\|p_{t+1}\|_2^2
\leq
a^2(1+|\cos\omega|)\|C\|_F^2.
\label{eq:osc-window-bounds}
\end{equation}
The lower bound is strictly positive because $0<\omega<\pi$,
$a>0$, and $C\neq0$.
The results show that $p_t$ can be temporarily collapsed close to 0 but $p_t$ and $p_{t+1}$ cannot become arbitrarily small at the same time. 

Extending this result to $e_t = r^tp_t + \xi_t$, define the two-step envelope
\begin{equation}
W_t:=
\left(\|e_t\|_2^2+r^{-2}\|e_{t+1}\|_2^2\right)^{1/2}.
\label{eq:osc-envelope-definition}
\end{equation}
Eq.~\ref{eq:osc-leading-response} gives
\[
\begin{pmatrix}
e_t\\
r^{-1}e_{t+1}
\end{pmatrix}
=
r^t
\begin{pmatrix}
p_t\\
p_{t+1}
\end{pmatrix}
+
\begin{pmatrix}
\xi_t\\
r^{-1}\xi_{t+1}
\end{pmatrix},
\]
where
\[
\left\|
\begin{pmatrix}
\xi_t\\
r^{-1}\xi_{t+1}
\end{pmatrix}
\right\|_2
=O(\nu^t)=o(r^t).
\]
The triangle and reverse triangle inequalities therefore imply
\[
\left|
W_t-r^t\left(\|p_t\|_2^2+\|p_{t+1}\|_2^2\right)^{1/2}
\right|
=O(\nu^t).
\]
Combining this with Eq.~\ref{eq:osc-window-bounds}, there exist constants
$c_-,c_+>0$ and $t_0$ such that
\begin{equation}
c_-r^t\leq W_t\leq c_+r^t,\qquad t\geq t_0.
\label{eq:osc-envelope}
\end{equation}
Now suppose $r<1$.
Since
\[
\|p_t\|_2\leq\|C\|_2a,
\]
Eq.~\ref{eq:osc-leading-response} implies $e_t\to0$ and
\[
\limsup_{t\to\infty}\|e_t\|_2^{1/t}\leq r.
\]
If this limsup were strictly smaller than $r$, there would exist
$0<s<r$ such that
\[
\|e_t\|_2=O(s^t).
\]
It would follow from Eq.~\ref{eq:osc-envelope-definition} that
\[
W_t=O(s^t),
\]
contradicting the lower bound $W_t\geq c_-r^t$ in
Eq.~\ref{eq:osc-envelope}.
Hence
\[
\limsup_{t\to\infty}\|e_t\|_2^{1/t}=r,
\]
which proves Eq.~\ref{eq:osc-inherited-rate}.
We next derive the settling-time asymptotic.
Since $e_t\to0$, $\tau_\delta(e)$ is finite for every $\delta>0$.
The positive lower bound in Eq.~\ref{eq:osc-envelope} rules out an
eventually zero trajectory, 
\[
\tau_\delta(e)\longrightarrow\infty
\qquad\text{as }\delta\downarrow0.
\]
Thus, for sufficiently small $\delta$, the relevant settling times
are at least $t_0$.
If $t\geq t_0$ satisfies
\[
c_+r^t\leq\delta,
\]
then, for every $k\geq t$,
\[
\|e_k\|_2
\leq W_k
\leq c_+r^k
\leq c_+r^t
\leq\delta.
\]
Therefore, for sufficiently small $\delta$,
\begin{equation}
\tau_\delta(e)
\leq
\left\lceil\frac{\log(c_+/\delta)}{-\log r}\right\rceil
\leq
1+\frac{\log(c_+/\delta)}{-\log r} = \frac{\log\!\left(1/\delta\right)}{-\log r} + O(1).
\label{eq:osc-settling-upper}
\end{equation}
Conversely for the lower bound, if $\|e_k\|_2\leq\delta$ for every $k\geq t$, then
\[
W_t\leq\delta\sqrt{1+r^{-2}}.
\]
At $t=\tau_\delta(e)\geq t_0$, the lower bound in
Eq.~\ref{eq:osc-envelope} therefore gives
\[
c_-r^{\tau_\delta(e)}
\leq
\delta\sqrt{1+r^{-2}},
\]
and hence
\begin{equation}
\tau_\delta(e)
\geq
\frac{\log\!\left(c_-/(\delta\sqrt{1+r^{-2}})\right)}{-\log r} = \frac{\log\!\left(1/\delta\right)}{-\log r} + O(1).
\label{eq:osc-settling-lower}
\end{equation}
Combining Eq.~\ref{eq:osc-settling-upper} and
Eq.~\ref{eq:osc-settling-lower} yields
\[
\tau_\delta(e)
=
\frac{\log(1/\delta)}{-\log r}+O(1),
\qquad
\delta\downarrow0,
\]
which proves Eq.~\ref{eq:osc-settling-asymptotic}. This shows that $e_t$ inherits the slow rate of the oscillatory mode with $r < 1$ due to the coupling.

Now it remains to consider $r=1$.
Eq.~\ref{eq:osc-leading-response} becomes
\begin{equation}
e_t=p_t+\xi_t,\qquad
p_t=CQ_\omega^tq,\qquad
\xi_t\to0.
\label{eq:osc-persistent-decomposition}
\end{equation}
The sequence $(e_t)$ is bounded because $(p_t)$ is bounded and
$\xi_t$ decays geometrically.
Write
\[
q=a(\cos\phi,\sin\phi)^\top.
\]
Since $0<\omega<\pi$, both $e^{i\omega}$ and $e^{2i\omega}$ differ
from $1$.
For $\ell\in\{1,2\}$, the geometric-sum formula gives
\[
\frac1T\sum_{t=0}^{T-1}e^{i\ell t\omega}
=
\frac{1-e^{i\ell T\omega}}{T(1-e^{i\ell\omega})}
\longrightarrow0.
\]
Applying the $\ell=1$ identity to
\[
Q_\omega^tq
=
a
\begin{pmatrix}
\cos(\phi+t\omega)\\
\sin(\phi+t\omega)
\end{pmatrix}
\]
yields zero-mean for the rotating vector $Q_\omega^tq$
\begin{equation}
\frac1T\sum_{t=0}^{T-1}Q_\omega^tq
\longrightarrow0.
\label{eq:osc-mean-zero}
\end{equation}
For the second moment, the double-angle identities give
\[
Q_\omega^tqq^\top(Q_\omega^t)^\top
=
\frac{a^2}{2}I_2
+
\frac{a^2}{2}
\begin{pmatrix}
\cos(2\phi+2t\omega) & \sin(2\phi+2t\omega)\\
\sin(2\phi+2t\omega) & -\cos(2\phi+2t\omega)
\end{pmatrix}.
\]
The $\ell=2$ identity therefore implies
\begin{equation}
\frac1T\sum_{t=0}^{T-1}
Q_\omega^tqq^\top(Q_\omega^t)^\top
\longrightarrow
\frac{a^2}{2}I_2.
\label{eq:osc-second-moment}
\end{equation}
These limits require no irrationality assumption on $\omega/(2\pi)$.
Since $p_t=CQ_\omega^tq$, we obtain
\[
\frac1T\sum_{t=0}^{T-1}p_t
=
C\left(\frac1T\sum_{t=0}^{T-1}Q_\omega^tq\right)
\longrightarrow0,
\]
and
\[
\frac1T\sum_{t=0}^{T-1}\|p_t\|_2^2
=
\operatorname{tr}\!\left[
C\left(
\frac1T\sum_{t=0}^{T-1}
Q_\omega^tqq^\top(Q_\omega^t)^\top
\right)C^\top
\right]
\longrightarrow
\frac{a^2}{2}\|C\|_F^2.
\]
Thus
\begin{equation}
\frac1T\sum_{t=0}^{T-1}p_t\longrightarrow0,
\qquad
\frac1T\sum_{t=0}^{T-1}\|p_t\|_2^2
\longrightarrow\frac{a^2}{2}\|C\|_F^2.
\label{eq:osc-p-response-moments}
\end{equation}
Fix any $c\in\mathbb R^d$.
The squared distance satisfies
\[
\|e_t-c\|_2^2
=
\|p_t-c\|_2^2
+
2\langle p_t-c,\xi_t\rangle
+
\|\xi_t\|_2^2.
\]
Because $\xi_t$ decays geometrically, the last two terms vanish as
\[
\frac1T\sum_{t=0}^{T-1}\|\xi_t\|_2\longrightarrow0,
\qquad
\frac1T\sum_{t=0}^{T-1}\|\xi_t\|_2^2\longrightarrow0.
\]
Moreover,
\[
\left|
\frac1T\sum_{t=0}^{T-1}\langle p_t-c,\xi_t\rangle
\right|
\leq
\left(\|C\|_2a+\|c\|_2\right)
\frac1T\sum_{t=0}^{T-1}\|\xi_t\|_2
\longrightarrow0.
\]
Using
\[
\|p_t-c\|_2^2
=
\|p_t\|_2^2-2\langle p_t,c\rangle+\|c\|_2^2,
\]
Eq.~\ref{eq:osc-p-response-moments} therefore yields
\begin{equation}
\lim_{T\to\infty}
\frac1T\sum_{t=0}^{T-1}\|e_t-c\|_2^2
=
\frac{a^2}{2}\|C\|_F^2+\|c\|_2^2.
\label{eq:osc-mean-square-radius}
\end{equation}
For every $c$, the limsup is at least the limiting Ces\`aro average,
\[
\limsup_{t\to\infty}\|e_t-c\|_2^2
\geq
\lim_{T\to\infty}
\frac1T\sum_{t=0}^{T-1}\|e_t-c\|_2^2.
\]
Taking square roots and then the infimum over $c$ gives
\begin{equation}
\mathcal E(e)
\geq
\frac{a}{\sqrt2}\|C\|_F
=
\frac{\|q\|_2\|C\|_F}{\sqrt2}
=
\frac{\|M^{-1}u_0\|_2\|PM\|_F}{\sqrt2}.
\label{eq:osc-response-floor}
\end{equation}
Finally, we express this bound in the original coordinates.
Multiplying $PS-AP=B$ on the right by $M$, and using
$SM=MQ_\omega$ when $r=1$, gives
\[
BM=PSM-APM=CQ_\omega-AC.
\]
Since $Q_\omega$ is orthogonal,
\[
\|CQ_\omega\|_F=\|C\|_F.
\]
Consequently,
\begin{align*}
\|BM\|_F
&\leq \|CQ_\omega\|_F+\|AC\|_F\\
&\leq \|C\|_F+\|A\|_2\|C\|_F\\
&=(1+\|A\|_2)\|C\|_F.
\end{align*}
Substituting into Eq.~\ref{eq:osc-response-floor} gives
\[
\mathcal E(e)
\geq
\frac{\|M^{-1}u_0\|_2\|BM\|_F}
{\sqrt2(1+\|A\|_2)}.
\]
The identities $u_0=M(M^{-1}u_0)$ and $B=(BM)M^{-1}$ imply
\[
\|u_0\|_2
\leq
\|M\|_2\|M^{-1}u_0\|_2
\]
and
\[
\|B\|_F
\leq
\|BM\|_F\|M^{-1}\|_2.
\]
Therefore,
\[
\|M^{-1}u_0\|_2\|BM\|_F
\geq
\frac{\|u_0\|_2\|B\|_F}
{\|M\|_2\|M^{-1}\|_2}
=
\frac{\|u_0\|_2\|B\|_F}{\kappa_2(M)}.
\]
We conclude that
\[
\mathcal E(e)
\geq
\frac{\|u_0\|_2\|B\|_F}
{\sqrt2\,\kappa_2(M)(1+\|A\|_2)}
>0,
\]
which proves Eq.~\ref{eq:osc-persistent-floor}.
Hence the complementary state remains bounded but cannot converge to
any constant.
\end{proof}

\begin{restatable}[Alternating modes]{corollary}{alternatingcoupling}
\label{cor:alternating-coupling}
Consider the alternating oscillatory mode
\begin{equation}
e_{t+1}=Ae_t+bu_t,\qquad u_{t+1}=-ru_t,
\label{eq:alternating-coupled}
\end{equation}
where $e_t,b\in\mathbb R^d$, $u_t\in\mathbb R$, $A \in \mathbb{R}^{d \times d}$, 
$\rho(A)<r\leq1$, and $bu_0\neq0$.
If $r<1$, the rate and settling-time conclusions from~\ref{thm:oscillatory-coupling}
hold. More precisely, we can further derive
\begin{equation}
\lim_{t\to\infty}\frac{\|e_t\|_2}{r^t}=\|(-rI-A)^{-1}bu_0\|_2>0.
\label{eq:alternating-amplitude}
\end{equation}
If $r=1$, the trajectory approaches an alternating two-cycle, and
\begin{equation}
\mathcal E(e)=\|(I+A)^{-1}bu_0\|_2>0.
\label{eq:alternating-floor}
\end{equation}
\end{restatable}

\begin{proof}
Since $\rho(A)<r$, the matrix $-rI-A$ is invertible.
Define
\begin{equation}
v=(-rI-A)^{-1}bu_0.
\end{equation}
Then $v\neq0$ and $-rv-Av=bu_0$.
Direct substitution into Eq.~\ref{eq:alternating-coupled}
gives
\begin{equation}
e_t=A^t(e_0-v)+(-r)^tv.
\label{eq:alternating-decomposition}
\end{equation}
The matrix-power bound implies
$A^t(e_0-v)=o(r^t)$.
Therefore,
\begin{equation}
r^{-t}e_t=(-1)^tv+o(1),
\end{equation}
which proves Eq.~\ref{eq:alternating-amplitude}.

For $r<1$, the norm is consequently bounded above and below
by positive multiples of $r^t$ for all sufficiently large
$t$.
The exponential-rate and settling-time conclusions follow
as in Theorem~\ref{thm:oscillatory-coupling}.

For $r=1$, Eq.~\ref{eq:alternating-decomposition} gives
\begin{equation}
e_{2t}\longrightarrow v,\qquad e_{2t+1}\longrightarrow-v.
\end{equation}
Thus
\begin{equation}
\mathcal E(e)=\inf_{c\in\mathbb R^d}\max\{\|v-c\|_2,\|-v-c\|_2\}=\|v\|_2.
\end{equation}
The last equality follows because any ball containing
$v$ and $-v$ has radius at least $\|v\|_2$, attained at
$c=0$.
Since $v=-(I+A)^{-1}bu_0$ when $r=1$, this proves
Eq.~\ref{eq:alternating-floor}.
\end{proof}

\paragraph{Perturbation response}
In the alternating case $S=-r$, the Sylvester equation gives
\[
P(-r)-AP=b,
\qquad
P=(-rI-A)^{-1}b.
\]
Hence the linearized complementary trajectory is
\[
\widehat e_k^{(\alpha)}
=
A^k e_0+\alpha\bigl(P(-r)^k-A^kP\bigr)u_0.
\]
Subtracting the $\alpha=0$ counterfactual isolates the response induced by the oscillatory coordinate,
\[
\widehat{\Delta e}_k^{(\alpha)}
:=
\widehat e_k^{(\alpha)}-\widehat e_k^{(0)}
=
\alpha\bigl(P(-r)^k-A^kP\bigr)u_0.
\]
Since $\rho(A)<r$, we have $A^kPu_0=o(r^k)$, and therefore
\[
\widehat{\Delta e}_k^{(\alpha)}
=
\alpha P(-r)^k u_0+o(r^k).
\]

\subsubsection{Structural non-forcing of gauge motion.}
\label{app:gauge_nonforcing}

The quotient construction in Section~\ref{sec:quotient_definition} is stronger than readout invariance, exact equivariance prevents symmetry-generated motion from forcing dynamics transverse to the gauge orbit.

\begin{restatable}[Gauge-to-quotient non-forcing]{proposition}{gaugenonforcing}
\label{prop:gauge-nonforcing}
Let $T$ be differentiable and equivariant under the token-local
$G=U(1)^S$ action. Define the gauge tangent space
$\mathcal V_z:=\operatorname{Tan}_z(G\cdot z)$, and let $\Pi_z$ denote the orthogonal
projector onto $\mathcal V_z^\perp$. Then
\begin{equation}
\Pi_{T(z)}DT(z)\xi=0,
\qquad
\xi\in\mathcal V_z.
\end{equation}
At a relative equilibrium $T(z^\star)=g^\star\cdot z^\star$,
define $\widehat T=(g^\star)^{-1}\cdot T$ and
$\widehat J=D\widehat T(z^\star)$. If $U_G$ has orthonormal
columns spanning $\mathcal{V}*{z^\star}$, then
\begin{equation}
\widehat J U_G=U_G,
\qquad
B_G:=\Pi*{z^\star}\widehat J U_G=0.
\label{eq:gauge-nonforcing-root}
\end{equation}
\end{restatable}

\begin{proof}
Let $\xi\in\mathcal{V}_z$. By definition of the tangent space to the
group orbit, there exists a differentiable curve $g(\phi)\in G$ with
$g(0)=e$ such that
\begin{equation}
\xi
=
\left.
\frac{d}{d\phi}
g(\phi)\cdot z
\right|_{\phi=0}.
\end{equation}
By equivariance,
\begin{equation}
T(g(\phi)\cdot z)
=
g(\phi)\cdot T(z).
\end{equation}
Differentiating at $\phi=0$ gives
\begin{equation}
DT(z)\xi
=
\left.
\frac{d}{d\phi}
g(\phi)\cdot T(z)
\right|_{\phi=0}
\in
\mathcal{V}_{T(z)}.
\end{equation}
Hence
\begin{equation}
DT(z)\mathcal{V}_z
\subseteq
\mathcal{V}_{T(z)},
\end{equation}
and therefore
\begin{equation}
\Pi_{T(z)}DT(z)\xi=0,
\qquad
\xi\in\mathcal{V}_z.
\end{equation}

Now let $z^\star$ satisfy
$T(z^\star)=g^\star\cdot z^\star$, and define
\begin{equation}
\widehat T=(g^\star)^{-1}\cdot T.
\end{equation}
Then $\widehat T(z^\star)=z^\star$. Since $G=U(1)^S$ is Abelian, $gg^\star=g^\star g$ for all
$g,g^\star\in G$, and hence
\begin{align}
\widehat T(g\cdot z)
&=
(g^\star)^{-1}\cdot T(g\cdot z)
\\
&=
(g^\star)^{-1}g\cdot T(z)
\\
&=
g(g^\star)^{-1}\cdot T(z)
=
g\cdot\widehat T(z).
\end{align}
Thus $\widehat T$ is also $G$-equivariant.

Thus, for any
$\xi\in\mathcal{V}_{z^\star}$ and any curve $g(\phi)$ generating $\xi$,
\begin{align}
\widehat J\xi
&=
\left.
\frac{d}{d\phi}
\widehat T(g(\phi)\cdot z^\star)
\right|_{\phi=0}
\\
&=
\left.
\frac{d}{d\phi}
g(\phi)\cdot\widehat T(z^\star)
\right|_{\phi=0}
\\
&=
\left.
\frac{d}{d\phi}
g(\phi)\cdot z^\star
\right|_{\phi=0}
=
\xi.
\end{align}
Therefore $\widehat J$ acts as the identity on
$\mathcal{V}_{z^\star}$. If $U_G$ spans this space,
\begin{equation}
\widehat J U_G=U_G.
\end{equation}
Thus, in the co-moving frame, while the raw state advances along the gauge orbit by $g^\star$, removing this common
rotation leaves infinitesimal gauge perturbations unchanged to first order.

In direct analogy with the coupling matrix
$B=QJU$ in Section~\ref{sec:theory}, define the coupling from the gauge coordinate
into the transverse quotient dynamics as
\begin{equation}
B_G
:=
\Pi_{z^\star}\widehat J U_G.
\end{equation}
Since $\Pi_{z^\star}U_G=0$ and $\widehat J U_G=U_G$,
\begin{equation}
B_G
=
\Pi_{z^\star}\widehat J U_G
=
\Pi_{z^\star}U_G
=
0.
\end{equation}

Finally, let $W$ be an orthonormal basis of
$\mathcal{V}_{z^\star}^{\perp}$ and write
$\delta=We+U_Gu$. Projecting the co-moving linearization onto the
quotient coordinates gives
\begin{align}
e^{+}
&=
W^\top\widehat J(We+U_Gu)
\\
&=
W^\top\widehat JW\,e
+
W^\top\widehat JU_G\,u
\\
&=
A_Qe,
\end{align}
where $A_Q:=W^\top\widehat JW$. Thus the quotient component evolves
independently of the gauge coordinate.
\end{proof}

To connect this result to the coupled-mode analysis in
Section~\ref{sec:oscillatory_coupling}, let $W$ be an orthonormal basis of
$\mathcal{V}_{z^\star}^{\perp}$ and write
$\delta=We+U_Gu$.
Since $W^\top\widehat J U_G=0$, the quotient component evolves locally as

$$
e_{k+1}=A_Qe_k,
\qquad
A_Q:=W^\top\widehat J W,
$$

independently of the gauge coordinate $u_k$.
Hence a symmetry-generated gauge rotation cannot transmit its timescale into the quotient dynamics.
This statement applies only to the gauge subspace so  other oscillatory modes transverse to the gauge orbit remain unconstrained and may continue to participate in the quotient dynamics.

\subsection{Local theory and empirical prediction}~\label{app:thm2}

\paragraph{Derivation of the local block dynamics.}~\label{app:localblock}

Let $U\in\mathbb R^{D\times m}$ have orthonormal columns and satisfy
\begin{equation}
J^\top U = US^\top.
\label{eq:left-invariant-subspace}
\end{equation}
Transposing gives
\begin{equation}
U^\top J = S U^\top.
\label{eq:left-invariant-transpose}
\end{equation}
Define
\begin{equation}
Q=I-UU^\top.
\end{equation}
Since $U^\top U=I$, $Q$ is the orthogonal projector onto the complement of
$\operatorname{span}(U)$, and
\begin{equation}
QU=0,
\qquad
U^\top Q=0.
\label{eq:projector-identities}
\end{equation}

For any root-centered displacement $\delta$, define
\begin{equation}
u=U^\top\delta,
\qquad
e=Q\delta.
\end{equation}
Then $\delta=e+Uu$, with $U^\top e=0$. Under the linearized dynamics
\begin{equation}
\delta^+=J\delta,
\end{equation}
the oscillatory coordinate evolves as
\begin{align}
u^+
&=U^\top\delta^+ \\
&=U^\top J(e+Uu) \\
&=S U^\top e + S U^\top Uu \\
&=Su.
\label{eq:local-u-dynamics}
\end{align}
Thus the complementary coordinate does not feed back into $u$ to first order.

Similarly,
\begin{align}
e^+
&=Q\delta^+ \\
&=QJ(e+Uu) \\
&=QJQ\,e+QJU\,u,
\end{align}
where we used $e=Qe$. Defining
\begin{equation}
A:=QJQ,
\qquad
B:=QJU,
\label{eq:local-AB}
\end{equation}
gives
\begin{equation}
e^+=Ae+Bu,
\qquad
u^+=Su.
\label{eq:local-triangular-system}
\end{equation}
Hence the left-invariant oscillatory coordinate evolves autonomously, while it can
still drive the complementary state through the coupling block $B$.

\subsection{QRR Architecture}
\label{app:full_arc}

QRR follows the recurrent-block organization of FPRM, but replaces the recurrent trunk with the gauge-equivariant construction described in Section~\ref{sec:gauge_architecture} and augments recurrence with momentum.

\paragraph{Input and recurrent state.}
Let $x$ denote the embedded task input, which is kept fixed and re-injected at every recurrent step.
The recurrent state consists of a latent state $y_t$ and a momentum state $m_t$, both transforming covariantly under the token-local $G=U(1)^S$ action.
In real coordinates, $y_t,m_t\in\mathbb{R}^{S\times 2d}$, with the fixed identification
$\mathbb{R}^{2d}\simeq\mathbb{C}^{d}$ used throughout the gauge-equivariant trunk.
For each input, we initialize $m_0=0$ and draw $y_0$ independently at random.
Specifically, each four-dimensional real cell of $y_0$ is sampled coordinate-wise from a truncated normal distribution and normalized to unit norm.

\paragraph{Gauge-equivariant recurrent trunk.}
At recurrent step $t$, the trunk receives the current state, carried momentum, and input,
\begin{equation}
z_t
=
\Phi_\theta(y_t,m_t;x),
\label{eq:full-qrr-trunk}
\end{equation}
where $\Phi_\theta$ is $G$-equivariant,
\begin{equation}
\Phi_\theta(g\!\cdot y,g\!\cdot m;x)
=
g\!\cdot\Phi_\theta(y,m;x).
\end{equation}
The trunk retains the pre-norm Transformer structure of FPRM, while token mixing, input injection, attention, and feed-forward transformations are replaced by their gauge-equivariant counterparts.
The input $x$ is injected through a covariant carrier, while the carried momentum $m_t$ is also fed back into the trunk.
The resulting output $z_t$ is a \emph{state proposal}, rather than the next recurrent state directly.

\paragraph{Momentum-carrying update.}
We define the proposal residual
\begin{equation}
F_t=z_t-y_t.
\label{eq:full-qrr-residual}
\end{equation}
The deployed recurrence applies momentum directly to this residual,
\begin{equation}
p_t
=
\beta m_t+(1-\beta)F_t,
\qquad
\beta=\sigma(b_{\mathrm H}),
\label{eq:full-qrr-momentum}
\end{equation}
where $b_{\mathrm H}$ is a learned global scalar.
Ignoring an equivalent gauge rephasing described below, the update is
\begin{equation}
\widetilde y_{t+1}=y_t+p_t,
\qquad
\widetilde m_{t+1}=p_t.
\label{eq:full-qrr-update}
\end{equation}

\paragraph{Gauge rephasing.}
The implementation additionally permits a state-dependent gauge rotation
\begin{equation}
a_t
=
\frac{\langle F_t,\mathcal{J}y_t\rangle}
{\|y_t\|^2+\epsilon},
\qquad
R_t=\exp(a_t\mathcal{J}),
\label{eq:full-qrr-rotation}
\end{equation}
with $a_t$ detached from the computational graph.
The stored recurrent state is then
\begin{equation}
y_{t+1}=R_t\widetilde y_{t+1},
\qquad
m_{t+1}=R_t\widetilde m_{t+1}.
\label{eq:full-qrr-rotated-update}
\end{equation}
Since $R_t$ is itself a gauge transformation, this operation doesn't change quotient state.

\paragraph{Readout and adaptive halting.}
Predictions are produced by a gauge-invariant readout
\begin{equation}
\widehat y_t
=
\mathcal{R}(y_t;x),
\qquad
\mathcal{R}(g\!\cdot y_t;x)
=
\mathcal{R}(y_t;x).
\end{equation}
Consequently, prediction is insensitive to motion confined to the gauge orbit.
Adaptive halting is based on the quotient residual $r_G$ from Equation~\ref{eq:gauge_residual}, computed from the joint state $(y_t,m_t)$ after phase alignment.
A trajectory halts once this residual remains below the task-specific threshold for the prescribed persistence window; otherwise recurrence continues until the rollout cap.

\paragraph{Recurrent training.}
The pair $(y_t,m_t)$ is carried across recurrent steps and across successive deep-supervision segments for the same input, and is reset only when a new example is started.
Within each segment, the model is unrolled for a task-specific number of recurrent steps, while gradients are propagated only through the final truncated-BPTT window. Earlier steps advance the recurrent state without storing gradients.
After each deep-supervision segment, the prediction loss is applied and the carried state is detached before the next segment.

\subsection{Additional Experimental Results}
\label{app:add_exps}


\subsubsection{Causal verification on HRM spiral dynamics}
\label{app:hrm-spiral-validation}

We additionally verify the local oscillatory-coupling prediction on HRM, whose
solution-root dynamics exhibit a qualitatively different oscillatory geometry from
the alternating mode analyzed for FPRM in Section~\ref{sec:theory}.
Here, the identified slow spectral component is a two-dimensional real invariant
block
\begin{equation}
S\in\mathbb{R}^{2\times2},
\qquad
\sigma(S)=\{re^{\pm i\omega}\},
\qquad
0<\omega<\pi,
\end{equation}
with $\rho(A)<r\le1$.
Thus, rather than alternating along a single direction, the slow coordinate rotates
while decaying toward the solution root.
We use exactly the same intervention protocol used for FPRM in main paper.
The local linearization from Eq.~\ref{eq:general-response-prediction} predicts
\begin{equation}
\widehat{\Delta e}_k^{(\alpha)}
=
\alpha QJ^kUu_0
=
\alpha\left(PS^k-A^kP\right)u_0,
\label{eq:hrm-response-prediction}
\end{equation}
where the full matrix power $S^k$ retains the rotation of the two-dimensional spiral mode.

\begin{figure*}[h]
\vspace{-0.5em}
    \centering
    \small
    \includegraphics[width=0.92\textwidth]{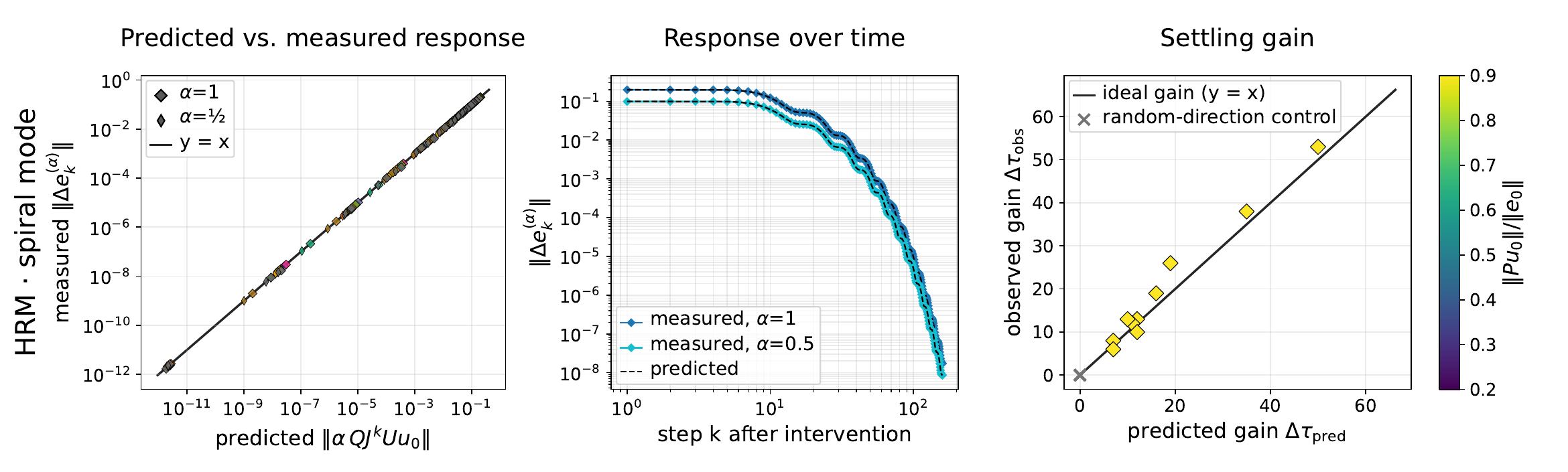}
    \vspace{-0.7em}
    \caption{
    \textbf{The local oscillatory-coupling prediction extends to the two-dimensional spiral dynamics of HRM.}
    We modify only the Jacobian-identified oscillatory coordinate once with
    $\alpha\in\{0,0.5,1\}$ and subsequently resume the original nonlinear HRM
    recurrence.
    }
    \label{fig:theoryfit-hrm}
\end{figure*}

Figure~\ref{fig:theoryfit-hrm} mirrors the three tests reported for FPRM in
Figure~\ref{fig:theoryfit}.
First, the response measured from the nonlinear rollout agrees with the
Jacobian-based prediction
\begin{equation}
\left\|
Q\!\left(
z_{t_0+k}^{(\alpha)}-z_{t_0+k}^{(0)}
\right)
\right\|_2
\qquad\text{vs.}\qquad
\left\|
\alpha QJ^kUu_0
\right\|_2.
\end{equation}
Second, the temporal response follows the predicted rotating and decaying evolution, rather than the one-dimensional alternation
observed in FPRM.
Finally, removing the oscillatory component with $\alpha=0$ reduces root-relative
settling time by an amount consistent with the local prediction, using the same
settling criterion and sustained-window implementation as in the main-text analysis.
These results show that the intervention agreement in
Figure~\ref{fig:theoryfit} is not specific to the nearly one-dimensional flip geometry of FPRM. The same local decomposition predicts the nonlinear response and convergence cost for a qualitatively distinct two-dimensional spiral mode.

\subsubsection{Spectral Analysis on HRM vs.\ QRR}
\label{app:spectral_bridge}

\begin{figure}[h]
    \centering
    \includegraphics[width=\linewidth]{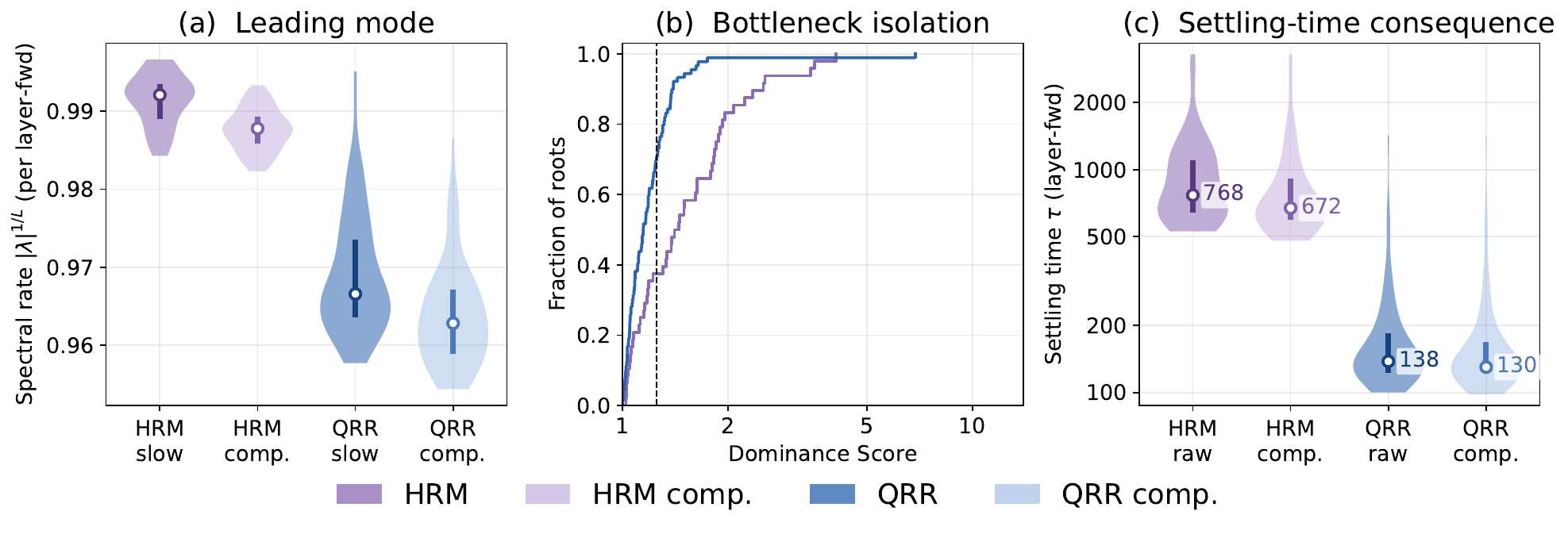}
    \caption{
    \textbf{Spectral organization and settling in HRM and QRR.}
    (a) Leading and complementary spectral rates near certified solution equilibria.
    QRR rates are measured in quotient space.
    (b) Distribution of the dominance score
    $D=\log r_{\mathrm{comp}}/\log r_{\mathrm{lead}}$.
    (c) Root-relative settling time before and after factoring out the leading-mode response.
    }
    \label{fig:spectral_story_hrm_qrr}
\end{figure}

We repeat the spectral analysis of Section~\ref{sec:spectral_bridge} for HRM, whose local dynamics exhibit the spiral geometry studied in Appendix~\ref{app:hrm-spiral-validation}. For HRM, we linearize the recurrent H-cycle map around certified solution fixed points. For QRR, we analyze the corresponding linearization in quotient space around relative equilibria. As in the main analysis, we decompose each system into its leading spectral mode and complementary dynamics and normalize spectral rates to recurrent layer-forward.

Figure~\ref{fig:spectral_story_hrm_qrr} shows the same distinction at a different absolute timescale. Although the leading--complementary spectral gaps appear similar in magnitude, HRM operates much closer to unit spectral rate, producing a larger separation in decay timescales and a stronger dominance score than QRR. This is reflected in finite-time settlement. Factoring out the leading-mode response reduces the median HRM settling time from $768$ to $672$ layer-forwards, whereas the corresponding change in the QRR quotient is only $138$ to $130$. Thus, QRR remains in a substantially faster regime and is less dependent on its leading mode for settlement, even relative to the qualitatively different spiral dynamics of HRM.

\subsubsection{Spectral Properties of Leading Oscillatory Modes}
~\label{app:empvalrho}

\vspace{-2em}
For each certified solution root, we linearize the one-step recurrent map and identify
the leading spectral mode of the local Jacobian $J$. We retain roots for which this
mode is oscillatory, corresponding to either a complex-conjugate pair or a real negative
flip mode. Let $U$ span this leading block, with modulus $r$, and define its complement
by $Q=I-UU^\top$ and $A=QJQ$. Because the selected block contains the leading
eigenmode of $J$, $r\geq\rho(A)$ holds by construction. We therefore use
$r-\rho(A)$ to measure the spectral separation between the leading oscillatory mode
and the remaining dynamics.

\begin{figure}[h]
    \centering
    \includegraphics[width=\linewidth]{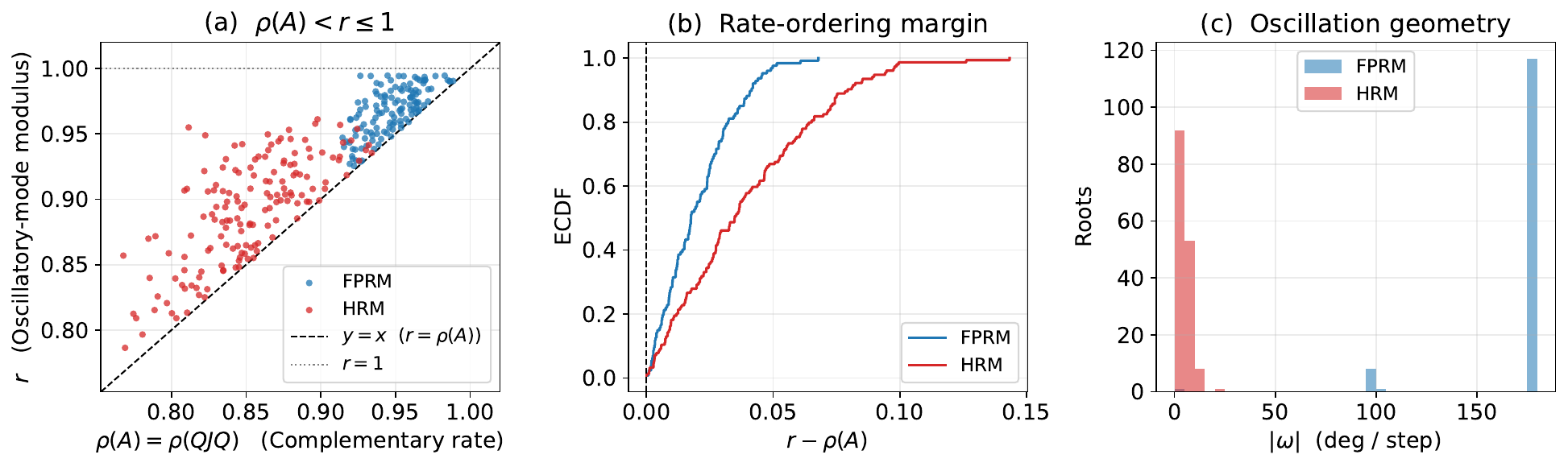}
    \caption{
    \textbf{Spectral separation of leading oscillatory modes.}
    {(a)} Leading-mode modulus $r$ and complementary spectral radius $\rho(A)$.
    {(b)} Spectral margin $r-\rho(A)$.
    {(c)} Oscillation angle $|\omega|$ for FPRM and HRM.
    }
    \label{fig:thm32-rate}
    \vspace{-0.7em}
\end{figure}

Figure~\ref{fig:thm32-rate} shows a strict spectral gap for all selected roots.
The median margin $r-\rho(A)$ is $0.0176$ for FPRM and $0.035$ for HRM,
with minimum margins of $2.2\times10^{-4}$ and $1.5\times10^{-4}$.
The leading modes also exhibit distinct geometries. FPRM is concentrated near
$|\omega|=\pi$ with median $179.7^\circ$, whereas HRM exhibits small-angle spirals
with median $3.9^\circ$. Thus, the selected oscillatory modes are cleanly separated
from the complementary spectrum across both flip and spiral regimes.

\subsubsection{FPRM flip mode suppression}
\label{app:fprm-osc-suppression}

The local analysis in Section~\ref{sec:oscillatory_coupling} shows that the slowly decaying flip mode can limit late-stage convergence in FPRM.
This does not imply, however, that the same mode is dispensable throughout reasoning.
We therefore test whether persistently suppressing the alternating component before the output has stabilized interferes with the computation itself.

For each trajectory, we estimate a one-dimensional alternating direction $U$ from the preceding recurrent updates and, from an intervention time $t_0$ onward, modify every subsequent update $d_t$ as
\begin{equation}
z_{t+1}^{(\alpha)}
=
z_t+
\left[I-(1-\alpha)UU^\top\right]d_t .
\label{eq:persistent-flip-suppression}
\end{equation}
We use $\alpha=\tfrac12$ as the suppression intervention and compare against the unmodified rollout.
Early interventions start before the native output-stabilization (decoded output stabilizes and remains equal for the remainder of the rollout), whereas late interventions start after stabilization.
To distinguish suppression of the oscillatory mode from a generic reduction in update magnitude, we additionally suppress a non-oscillatory control direction while matching the removed energy.
All comparisons are paired against the corresponding original rollout, and confidence intervals are obtained by bootstrapping over boards.

\begin{table}[h]
\centering
\small
\setlength{\tabcolsep}{5pt}
\renewcommand{\arraystretch}{1.08}
\caption{
\textbf{Persistent suppression of the FPRM flip mode.}
Suppressing the alternating component before output stabilization reduces final exact accuracy, whereas an energy-matched non-oscillatory intervention does not show a significant effect.
}
\label{tab:fprm-osc-suppression}
\begin{tabular}{lcc}
\toprule
{Intervention}
& $\boldsymbol{\Delta}$ {final accuracy}
& {95\% CI} \\
\midrule
Oscillatory suppression, early
& {$-5.9$ pp}
& $[-13.1,-0.6]$ \\

Energy-matched control, early
& $-1.1$ pp
& $[-2.8,0.0]$ \\

Oscillatory suppression, late
& $0.0$ pp
& $[0.0,0.0]$ \\
\bottomrule
\end{tabular}
\end{table}

Early suppression decreases final exact accuracy by $5.9$ percentage points, and delays output stabilization.
In contrast, removing a comparable amount of update energy along a non-oscillatory direction produces no statistically distinguishable degradation, and suppressing the flip mode only after the output has stabilized has no measurable effect.
These results indicate that the oscillatory motion of the same alternating form can still participate in the earlier reasoning computation.
This motivates QRR to avoid suppressing oscillations based on their dynamical appearance alone, and instead quotient out only motion whose representational redundancy is guaranteed by symmetry.

\subsubsection{Halting time sensitivity to task difficulty}
\label{app:difficulty_halting}

\begin{figure}[h]
    \vspace{-0.5em}
    \centering
    \includegraphics[width=0.9\linewidth]{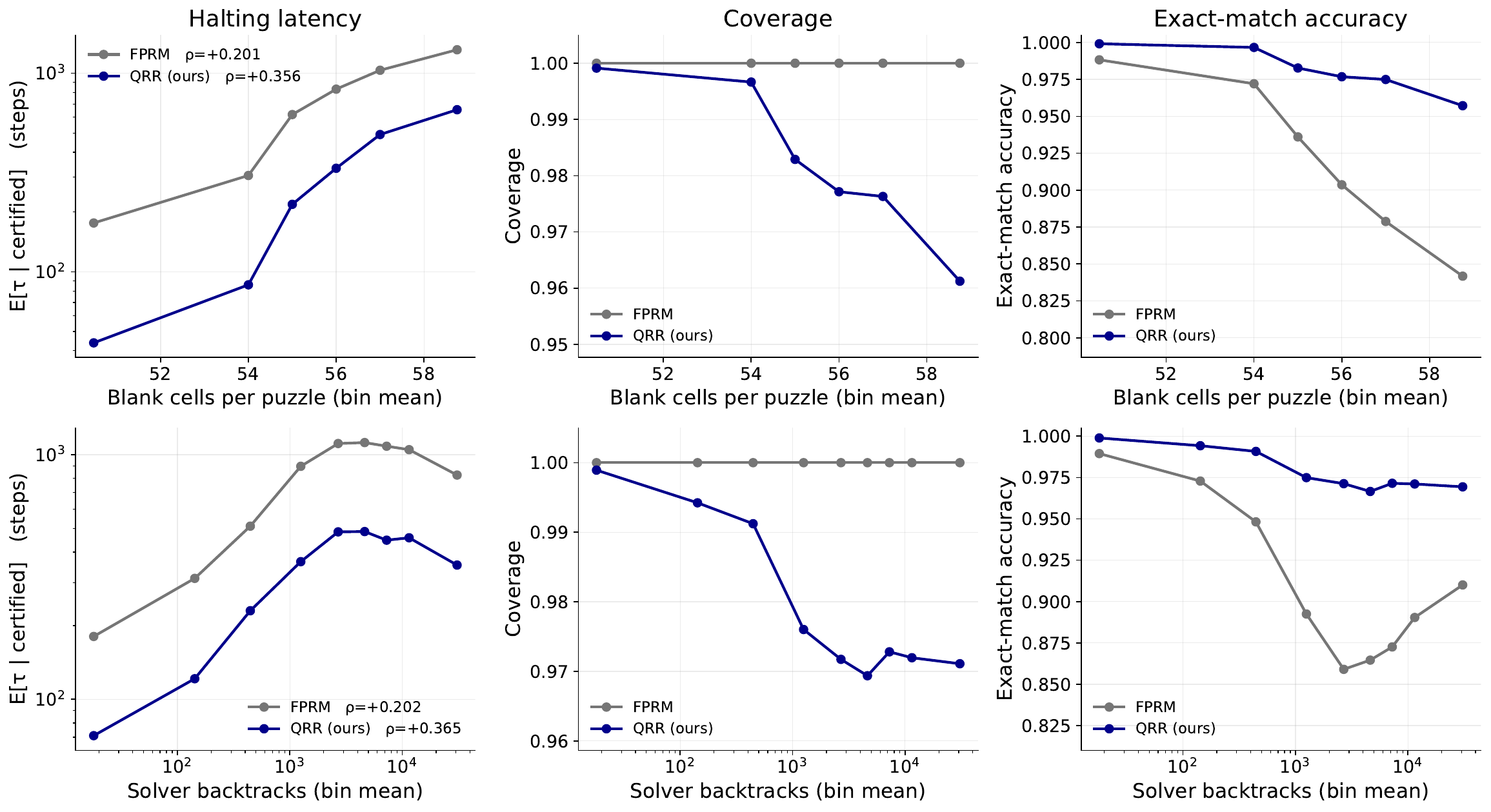}
    \vspace{-0.2em}
    \caption{
    \textbf{Adaptive halting as a function of Sudoku difficulty.}
    Difficulty as the number of blank cells (top) and the backtrack count (bottom).
    Left: mean halting time among certified examples.
    Middle: fraction of examples certified within the rollout budget.
    Right: exact-match accuracy.
    }
    \label{fig:difficulty-halting}
    \vspace{-1em}
\end{figure}

On $42{,}023$ held-out Sudoku-Extreme boards, we evaluate QRR and FPRM under the same halting rules and $35{,}000$-step cap as in Table~\ref{tab:convergence_halting}. As shown in Figure~\ref{fig:difficulty-halting}, halting latency increases with difficulty for both models, but more strongly for QRR. The Spearman correlation is $\rho=0.356$ versus $0.201$ for blank count and $0.365$ versus $0.202$ for solver backtracks. QRR nevertheless certifies significantly earlier across the difficulty range while maintaining $98.0\%$ overall coverage and consistently higher exact-match accuracy, indicating stronger difficulty-adaptive compute allocation with lower cost.

\subsection{Experimental Setup}
\label{app:exp_setup}

\subsubsection{Reasoning Tasks}
\label{app:datasets}

We evaluate on four reasoning benchmarks. \textbf{Sudoku-Extreme} consists of partially filled $9\times9$ Sudoku boards with unique solutions and is evaluated by exact-match accuracy of the completed grid~\citep{wang2025hierarchicalreasoningmodel}. \textbf{Maze-Hard} contains $30\times30$ shortest-path problems and is evaluated by exact match to the supervised solution path~\citep{wang2025hierarchicalreasoningmodel}. Because multiple optimal paths may exist in Maze-Hard, we additionally use \textbf{Maze-Unique}, constructed from perfect tree-structured mazes so that each start--goal pair has a unique solution~\citep{MiniARC_NeurIPSW2022,huang2026equilibriumreasonerslearningattractors}. \textbf{Mini-ARC} is a compact $5\times5$ abstract rule-induction benchmark in which the model infers a transformation from demonstration pairs and applies it to a held-out input~\citep{MiniARC_NeurIPSW2022,huang2026equilibriumreasonerslearningattractors}. We follow the standard augmentation-based evaluation, map predictions from each augmented view back to the original frame, aggregate identical candidate grids by voting, and report task-level Pass@1 from the highest-ranked prediction.

\subsubsection{Root convergence analysis}
\label{app:rootconvanalys}

Here we provide detailed protocol for root convergence analysis used in this paper, which is used for the local Jacobian and spectral analyses in Sections~\ref{sec:oscillatory_coupling}, ~\ref{sec:spectral_bridge}, and Table~\ref{tab:convergence_halting} (a).
We firstly identify solution equilibria using Anderson acceleration with history $m=6$,
Tikhonov regularization $10^{-8}$, and $\beta=1$ in fp64.
The solver runs for a fixed budget and returns the iterate with the smallest residual.
For FPRM, HRM, and EqR, the solved maps are one FPRM recurrent step, two HRM
H-cycles, and three EqR cycles, respectively. For QRR, Anderson acceleration is
applied to the gauge-aligned map. 

For each certified solution root, we run the original model from its native initialization and measure when the trajectory reaches the corresponding root. We define $\tau_{\mathrm{root}}$ as the first step after which the root-relative distance remains below $\delta=10^{-3}$ for a persistence window of approximately $96$ recurrent layer-forwards. This corresponds to $48$ native steps for FPRM and QRR and $4$ trajectory steps for HRM. For EqR, whose native step contains $42$ layer-forwards, we use two steps, corresponding to $84$ layer-forwards. \emph{Correct root found} denotes the fraction of evaluation examples for which the solver returns a certified solution equilibrium decoding to the ground truth. 

To compare convergence across architectures with different computation per native recurrent step, we report convergence times in recurrent layer-forward equivalents (\textsc{lfwd}). An \textsc{lfwd} counts one forward evaluation of one recurrent network layer. If one native step of model $m$ evaluates $\ell_m$ recurrent layers, then
\[
\tau_{\mathrm{lfwd}}^{(m)}
=
\ell_m\,\tau_{\mathrm{native}}^{(m)},
\]
where $\tau_{\mathrm{native}}^{(m)}$ is the number of native recurrent steps required to reach the certified solution equilibrium. Table~\ref{tab:lfwd_conversion} summarizes the conversion used for each model.

\begin{table*}[h]
    \centering
    \small
    \caption{Conversion from native recurrent steps to recurrent layer-forward equivalents.
    An \textsc{lfwd} counts one forward evaluation of one recurrent network layer.}
    \label{tab:lfwd_conversion}
    \begin{tabular}{l l r}
        \toprule
        Model & Layer-forward computation per native step & $\ell_m$ (\textsc{lfwd}/step) \\
        \midrule
        QRR
        & $1\times L_{\mathrm{layers}} = 1\times 2$
        & $2$ \\
        FPRM
        & $1\times L_{\mathrm{layers}} = 1\times 2$
        & $2$ \\
        HRM
        & $H_{\mathrm{cyc}}
           \bigl(L_{\mathrm{cyc}}L_{\mathrm{layers}} + H_{\mathrm{layers}}\bigr)
           =2(2\cdot4+4)$
        & $24$ \\
        EqR
        & $H_{\mathrm{cyc}}(L_{\mathrm{cyc}}+1)L_{\mathrm{layers}}
           =3(6+1)2$
        & $42$ \\
        \bottomrule
    \end{tabular}
\end{table*}

\subsubsection{Hyperparameters}~\label{hparams}

We use the same core QRR architecture across all tasks and tune only the
task-specific training, recurrent-compute, and inference settings summarized in
Table~\ref{tab:qrr_hparams}. Architectural details shared across tasks are given in
Section~\ref{app:full_arc}.

\begin{table}[h]
\centering
\small
\setlength{\tabcolsep}{4.5pt}
\renewcommand{\arraystretch}{1.10}
\caption{
\textbf{Task-specific training and inference configurations for QRR.}
The core architecture is shared across tasks. The reported result is from the best in-run evaluation checkpoint.
}
\label{tab:qrr_hparams}

\begin{tabular}{lcccc}
\toprule
\textbf{Setting}
& {Sudoku-Extreme}
& {Maze-Hard}
& {Maze-Unique}
& {Mini-ARC} \\
\midrule

\multicolumn{5}{l}{\textit{Data and optimization}} \\[0.1em]

Training data
& 1K $\times$ 1000 aug.
& 1K, no aug.
& 1K, no aug.
& 149 tasks, 100 aug. \\

Training epochs
& 60K
& 60K
& 6K
& 100K \\

Global batch size
& 768
& 768
& 384
& 768 \\

Learning rate
& $3{\times}10^{-4}$
& $1{\times}10^{-4}$
& $2{\times}10^{-4}$
& $6{\times}10^{-4}$ \\

Weight decay
& $10^{-3}$
& $1.0$
& $1.0$
& $0.1$ \\

Gradient clipping
& $1.0$
& $1.0$
& $1.0$
& $1.0$ \\

\midrule
\multicolumn{5}{l}{\textit{Compute setup}} \\[0.1em]

Training GPUs
& $4\times$ H100
& $4\times$ B200
& $4\times$ H100
& $1\times$ A100 \\

Parallelism
& DDP $\times 4$
& DDP $\times 4$
& DDP $\times 4$
& single GPU \\

Per-GPU batch size
& 192
& 192
& 96
& 768 \\

\midrule
\multicolumn{5}{l}{\textit{Recurrence and deep supervision}} \\[0.1em]

Max DS segment
& 12
& 24
& 24
& 8 \\

Recurrent steps / DS segment
& 6
& 6
& 6
& 6 \\

TBPTT depth
& 6
& 6
& 6
& 6 \\

\midrule
\multicolumn{5}{l}{\textit{Inference and halting}} \\[0.1em]

Halting criterion
& $r_G < 0.03$
& $r_G < 0.075$
& $r_G < 0.15$
& $r_G < 0.08$ \\

Persistence window $W$
& 1
& 1
& 1
& 2 \\

Rollout cap
& 35,000
& 3,000
& 3,000
& 100 \\

Evaluation
& halt-or-cap
& halt-or-cap
& halt-or-cap
& halt-or-cap + vote \\

\bottomrule
\end{tabular}
\end{table}

\end{document}